%% file: main.tex
\documentclass[11pt]{article}
\usepackage[colorlinks=true,linkcolor=blue,urlcolor=black,citecolor=blue,breaklinks]{hyperref}
\usepackage{color}
\usepackage{enumerate}
\usepackage{graphicx}
\usepackage{varwidth}
\usepackage{mathrsfs}
\usepackage{mathtools}
\usepackage{caption}
\usepackage{subcaption}
\usepackage{overpic}
\usepackage{multirow}
\usepackage{fullpage}
\usepackage{authblk}

\usepackage[round]{natbib}
\bibpunct{(}{)}{;}{a}{,}{,}

\usepackage{amsmath,amsthm,amssymb,colonequals,etoolbox}
\usepackage{thmtools}
\input{commands}

\title{The role of optimization geometry in single neuron learning}
\author[1]{Nicholas M.\ Boffi}
\author[2]{Stephen Tu}
\author[2,3]{Jean-Jacques E.\ Slotine}
\affil[1]{Courant Institute of Mathematical Sciences, New York University}
\affil[2]{Google Brain Robotics}
\affil[3]{Nonlinear Systems Laboratory, Massachusetts Institute of Technology}
\date{\today}

\begin{document}
\maketitle

\begin{abstract}
\input{abs}
\end{abstract}

\section{Introduction}
\input{intro}

\section{Problem setting and background}
\input{setup}

\section{Continuous-time theory}
\input{cont}

\section{Discrete-time algorithms}
\input{disc}

\section{Stochastic optimization}
\input{online}

\section{Experiments}
\input{sims}

\section{Conclusions and future directions}
\input{conc}

\bibliography{refs}
\clearpage
\onecolumn
\appendix

\section{Details on experimental setup}
\input{sim_details}

\section{Preliminary results}
\input{prior_rslts}

\section{Omitted proofs}
\input{proofs}

\end{document}

%% file: commands.tex
\newcommand{\calW}{\mathcal{W}}
\newcommand{\E}{\mathbb{E}}
\newcommand{\R}{\mathbb{R}}
\newcommand{\calR}{\mathcal{R}}
\newcommand{\calX}{\mathcal{X}}
\newcommand{\diam}{\mathsf{Diam}}
\newcommand{\calM}{\mathcal{M}}
\newcommand{\calC}{\mathcal{C}}

\newcommand{\calD}{\mathcal{D}}
\newcommand{\calF}{\mathcal{F}}
\newcommand{\wh}{\widehat}

\newcommand{\bw}{\mathbf{w}}
\newcommand{\bphi}{\boldsymbol\phi}

\newcommand{\bx}{\mathbf{x}}

\newcommand{\bz}{\mathbf{z}}

\newcommand{\bu}{\mathbf{u}}

\newcommand{\bThet}{\boldsymbol\Theta}
\newcommand{\bthet}{\boldsymbol\theta}

\newcommand{\bH}{\mathbf{H}}

\newcommand{\bI}{\mathbf{I}}

\newcommand{\by}{\mathbf{y}}

\newcommand{\err}{\text{err}}
\DeclareMathOperator*{\argmin}{arg\,min}

\DeclareMathOperator*{\Tr}{Tr}
\newcommand{\bX}{\mathbf{X}}
\newcommand{\ip}[2]{\left\langle #1, #2 \right\rangle}
\newcommand{\T}{\mathsf{T}}

\newcommand{\bregd}[2]{d_\psi\left(#1\ \middle\|\ #2\right)}
\newcommand{\bregdd}[2]{d_{\psi^*}\left(\nabla\psi\left(#1\right)\ \middle\|\ \nabla\psi\left(#2\right)\right)}

\newcommand{\dkl}[2]{d_{KL}\left(#1\ \middle\|\ #2\right)}

\newcommand{\norm}[1]{\left\Vert#1\right\Vert}
\newcommand{\abs}[1]{\left|#1\right|}

\newcommand{\Exp}[1]{\mathbb{E}\left[#1\right]}
\newcommand{\Expsub}[2]{\mathbb{E}_{#1}\left[#2\right]}

\newtheorem{thm}{Theorem}[section]
\newtheorem{lem}{Lemma}[section]
\newtheorem{cor}{Corollary}[section]

\theoremstyle{definition}

\theoremstyle{definition}

\theoremstyle{definition}

\theoremstyle{definition}
\newtheorem{assumption}{Assumption}[section]

\DeclareMathOperator{\arcsinh}{arcsinh}

%% file: abs.tex
Recent numerical experiments have demonstrated that the choice of optimization geometry used during training can impact generalization performance when learning expressive nonlinear model classes such as deep neural networks. 
These observations have important implications for modern deep learning but remain poorly understood due to the difficulty of the associated nonconvex optimization problem.
Towards an understanding of this phenomenon, 
we analyze a family of pseudogradient methods for learning generalized linear models under the square loss -- a simplified problem containing both nonlinearity in the model parameters and nonconvexity of the optimization which admits a single neuron as a special case.
We prove non-asymptotic bounds on the generalization error that sharply characterize how the interplay between the optimization geometry and the feature space geometry sets the out-of-sample performance of the learned model.
Experimentally, selecting the optimization geometry as suggested by our theory leads to improved performance in generalized linear model estimation problems such as nonlinear and nonconvex variants of sparse vector recovery and low-rank matrix sensing.

%% file: intro.tex
Optimization geometry, whereby the loss gradient is computed with respect to a non-Euclidean metric, is a common tool in modern machine learning. Notable examples of algorithms that use non-Euclidean metrics to improve convergence include mirror descent~\citep{beck_teb, nem_yud, krich}, natural gradient descent~\citep{amari, gunasekar2020mirrorless}, and adaptive gradient methods such as AdaGrad~\citep{adagrad} and Adam~\citep{adam}. 

Recently, \citet{azizan_2} showed empirically that varying the optimization geometry through choice of a mirror descent potential can improve the generalization performance of expressive model classes such as deep neural networks, but a theoretical characterization of when and why this will occur is currently absent. 
One path towards explaining these observations could be to consider the effect of mirror descent in linear regression. However, nonconvexity of the optimization problem and nonlinearity in the parameters are hallmarks of deep networks, and it is not clear that conclusions about linear models lacking these properties will carry over to deep learning. 

The simplest model class containing both these characteristics is the class of generalized linear models (GLMs). GLMs extend linear models by incorporating a single ``layer'' of nonlinearity~\citep{GLMs}, wherein the dependent variables are assumed to be given as a known activation function of a linear predictor of the data. From a modern perspective, a GLM represents a single neuron, and as such can be seen as one of the most elementary models of a neural network that allows for rigorous analysis yet still contains both nonconvexity and nonlinearity. As a result, guarantees for GLMs may provide insight into the theoretical properties of more complex models, an observation that has been exploited by several recent works~\citep{maillard2021landscape, barbier5451, agnostic_neuron}.

In this work, we revisit the GLM-tron of \citet{glmtron}.
Leveraging recent developments in continuous-time optimization and adaptive control theory~\citep{boffi2019higherorder}, we
extend the continuous-time limit of the GLM-tron iteration 
to a mirror descent-like flow that we call the Reflectron. The Reflectron is specified by the choice of a pseudogradient $\xi$ and a strongly convex potential function $\psi$; these two ingredients together define a search direction and a search geometry. As particular cases, the Reflectron recovers the GLM-tron, gradient descent, and mirror descent.

\paragraph{Outline of results} We first prove non-asymptotic generalization error bounds for the continuous-time Reflectron in the full-batch setting. 
Our results highlight how the choice of $\psi$ can improve the \textit{statistical performance} of the model if selected in a way that respects the underlying geometric structure of the feature space. In the realizable setting, we further characterize the learned parameters as minimizing the Bregman divergence under $\psi$ between the initialization and the interpolating manifold.

We subsequently discretize the continuous dynamics via the forward-Euler method to obtain an implementable algorithm with matching guarantees. We further consider a stochastic gradient-like setting for learning GLMs in the realizable and bounded noise settings, and prove $\mathcal{O}\left(1/t\right)$ and $\mathcal{O}\left(1/\sqrt{t}\right)$ bounds for the generalization error, respectively. 

We conclude with experiments highlighting the ability of our theoretical results to capture the importance of optimization geometry in practice. We study two model problems that amount to GLM estimation under the square loss: nonlinear and nonconvex variants of sparse vector recovery and low-rank matrix sensing. By choosing the mirror descent potential as suggested by our analysis, we demonstrate improved generalization performance of the learned model.

\subsection{Related work and significance}

\paragraph{Applications of the GLM-Tron} 
The GLM-tron of~\citet{glmtron} was the first computationally and statistically efficient algorithm for learning both GLMs and Single Index Models (SIMs).
A recent extension known as the BregmanTron~\citep{bregmantron} obtains improved guarantees for the SIM problem by applying Bregman divergences to directly learn the loss function; here, we instead focus on the GLM-tron as an algorithmic primitive.
~\citet{agnostic_neuron} use similar proof techniques to~\citet{glmtron} to analyze gradient descent on the square loss for learning a single neuron. Our work extends their results to the mirror descent and pseudogradient settings to characterize the impact of optimization geometry in a similar class of problems.~\citet{foster2020learning} utilize the GLM-tron for system identification in a particular nonlinear discrete-time dynamics model, and~\citet{alphatron} use a kernelized GLM-tron to provably learn two-hidden-layer neural networks. Similar update laws have independently been developed in the adaptive control literature~\citep{tyukin_paper}, along with mirror descent and momentum variants~\citep{boffi2019higherorder}.

\paragraph{Implicit bias and generalization} 
Modern machine learning frequently takes place in a high-dimensional regime with many more parameters than examples. It is now well-known that deep networks will interpolate noisy data, yet exhibit low generalization error \textit{despite interpolation} when trained on meaningful data~\citep{und_dl}. Defying classical statistical wisdom, an explanation for this apparent paradox has been given in the \textit{implicit bias}~\citep{gunasekar_1} of optimization algorithms and the double descent curve~\citep{belkin, barlett, sahai, hastie2019surprises}. The notion of implicit bias captures the proclivity of a method to converge to a particular kind of interpolating solution -- such as minimum norm -- when many options exist.

Implicit bias has been categorized for gradient-based algorithms on separable classification problems~\citep{gunasekar_1, gunasekar_4}, regression problems~\citep{gunasekar_2}, and multilayer models~\citep{gunasekar_3, gunasekar_5, gunasekar_6}. Approximate results are also available for the implicit bias of gradient-based algorithms when used to train deep networks~\citep{azizan_1}. Moreover, it was shown empirically~\citep{azizan_2} that the choice of mirror descent potential affects the generalization error of deep networks, and a qualitative explanation was provided in terms of changing the implicit bias. Our focus on GLMs allows us to quantify the generalization performance of models trained with different potentials, which provides an analytical and geometric explanation for this observation. While we focus on the square loss (rather than cross entropy), a number of recent works have investigated the possibility of using the square loss for training deep networks for classification~\citep{mse_1, mse_2, mse_3}.

%% file: setup.tex
Our problem setting follows the original work of~\citet{glmtron}. Let $\{\bx_i, y_i\}_{i=1}^n$ denote an i.i.d.\ dataset sampled from a distribution $\mathcal{D}$ supported on $\mathcal{X}\times [0, 1]$, $\calX \subseteq \R^d$, where $\Exp{y_i|\bx_i} = u\left(\ip{\bthet}{\bx_i}\right)$ for $\bthet\in\mathbb{R}^d$ a fixed, unknown vector of parameters. $u: \mathbb{R}\rightarrow [0, 1]$ is assumed to be a known, nondecreasing, and $L$-Lipschitz activation function. Our goal is to approximate $\Exp{y_i|\bx_i}$ as measured by the expected square loss. To this end, for a hypothesis $h : \R^d \rightarrow \R$, we define the generalization error $\err(h)$ and the excess risk compared to the Bayes-optimal predictor $\varepsilon(h)$ as
\begin{align}
    \label{eqn:err_1}
    \err(h) &= \Expsub{\bx, y}{\left(h\left(\bx\right) - y\right)^2},\\
    \label{eqn:err_2}
    \varepsilon(h) &= \Expsub{\bx, y}{\left(h(\bx) - u\left(\ip{\bthet}{\bx}\right)\right)^2},
\end{align}
with $\wh{\err}(h)$ and $\wh{\varepsilon}(h)$ their empirical counterparts over the dataset. Towards minimizing $\err(h)$, we present a family of mirror descent-like algorithms for minimizing $\wh{\varepsilon}(h)$ over parametric hypotheses of the form $h(\bx) = u\left(\ip{\wh{\bthet}}{\bx}\right)$. Via standard statistical techniques~\citep{bartlett_gen}, we transfer our guarantees on $\wh{\varepsilon}(h)$ to $\varepsilon(h)$, which in turn implies a small $\err(h)$. The starting point of our analysis is the GLM-tron of~\citet{glmtron}, which is an iterative update law of the form
\begin{align}
    \label{eqn:glm}
    \wh{\bthet}_{t+1} &= \wh{\bthet}_t - \frac{1}{n}\sum_{i=1}^n\left(u\left(\ip{\wh{\bthet}}{\bx_i}\right) - y_i\right)\bx_i,
\end{align}
with initialization $\wh{\bthet}_1 = \mathbf{0}$. Equation~\eqref{eqn:glm} is a pseudogradient-based update law obtained from gradient descent on the square loss $\wh{\err}(h)$ by replacing all appearances of $u'$ by the fixed value $1$. It admits a continuous-time limit for an infinitesimal step size,
\begin{align}
    \label{eqn:glm_cont}
    \frac{d}{dt}\wh{\bthet} &= -\frac{1}{n}\sum_{i=1}^n\left(u\left(\ip{\wh{\bthet}}{\bx_i}\right) - y_i\right)\bx_i,
\end{align}
where (\ref{eqn:glm}) is obtained from (\ref{eqn:glm_cont}) via a forward-Euler discretization with a timestep $\Delta t = 1$. 

\paragraph{Notation.} Throughout this paper, we will use the notation $\frac{d}{dt}\bx = \dot{\bx}$ interchangeably for any time-dependent function $\bx(t)$. Moreover, we will denote
\begin{equation*}
    \mathcal{R}_n\left(\mathcal{F}\right) = \Expsub{\bx_i, \epsilon_i}{\sup_{h\in\calF}\frac{1}{n}\sum_{i=1}^n\epsilon_i h(\bx_i)}
\end{equation*}
the Rademacher complexity of a function class $\calF$ on $n$ samples, and the shorthand $\zeta(h) = \max\{\varepsilon(h), \wh{\varepsilon}(h)\}$ in our generalization error bounds.

%% file: cont.tex
\label{sec:cont}
In this section, we analyze a continuous-time flow that we will discretize to obtain implementable algorithms in Section~\ref{sec:disc}. Our continuous-time analysis sketches the essence of the techniques required to obtain discrete-time guarantees, and provides intuition for our main results while avoiding discretization-specific details. The class of algorithms we consider is captured by the dynamics
\begin{align}
    \label{eqn:ref}
    \frac{d}{dt}\nabla\psi\left(\wh{\bthet}\right) = -\frac{1}{n}\sum_{i=1}^n\left(u\left(\ip{\wh{\bthet}}{\bx_i}\right) - y_i\right)\xi\left(\wh{\bthet}, \bx_i\right)\bx_i.
\end{align}
for $\psi: \calM \rightarrow \mathbb{R}$, $\calM \subseteq \mathbb{R}^d$, and $\xi : \calM \times \calX \rightarrow \R$ with $\xi \geq 0$. To obtain guarantees on the algorithms represented by \eqref{eqn:ref}, we require two primary assumptions.
\begin{assumption}
\label{assmp:sc}
$\psi : \calM \rightarrow \R$ is $\sigma$-strongly convex with respect to a norm $\norm{\cdot}$. Moreover, $\min_{\bw\in\calM}\psi(\bw) = 0$.
\end{assumption}
Note that any $\psi$ with finite minimum can be shifted to satisfy the final requirement of Assumption~\ref{assmp:sc}, as our algorithms only depend on gradients and Bregman divergences of $\psi$.
\begin{assumption}
\label{assmp:link}
The activation function $u : \R\rightarrow [0, 1]$ is known, nondecreasing, and $L$-Lipschitz.
\end{assumption}
The parameters of the hypothesis $h_t$ at time $t$ are computed by applying the inverse gradient of $\psi$, which is guaranteed to exist by strong convexity. The mirror descent generalization of the GLM-tron is obtained from \eqref{eqn:ref} by setting $\xi\left(\bw, \bx\right) = 1$, while mirror descent itself is obtained by setting $\xi\left(\bw, \bx\right) = u'\left(\ip{\bw}{\bx}\right)$.
In order to outline the intuition 
behind our results, we focus exclusively on the case when
$\xi\left(\bw, \bx\right) = 1$ and  
defer the analysis with arbitrary $\xi$ to discrete-time.

\subsection{Statistical guarantees}
\label{sec:stat}
The following theorem gives a statistical guarantee for the Reflectron in continuous-time.
It shows that for any choice of potential function $\psi$, the Reflectron eventually finds a nearly Bayes-optimal predictor.
\begin{thm}
\label{thm:stat}
Suppose that $\{\bx_i, y_i\}_{i=1}^n$ are drawn i.i.d.~from a distribution $\mathcal{D}$ supported on $\mathcal{X}\times[0, 1]$ where $\mathbb{E}\left[y|\bx\right] = u\left(\ip{\bthet}{\bx}\right)$, $u$ satisfies Assumption~\ref{assmp:link}, and $\bthet \in \R^d$ is an unknown vector of parameters. Let $\psi$ satisfy Assumption~\ref{assmp:sc}. Assume that $\norm{\frac{1}{n} \sum_{i=1}^n \left(u\left(\ip{\bthet}{\bx_i}\right) - y_i\right)\bx_i}_* \leq \eta$ where $\Vert\cdot\Vert_*$ denotes the dual norm
to $\norm{\cdot}$. Then for any $\delta \in (0, 1)$, there exists some time $t < \sqrt{\frac{\psi(\bthet)\sigma}{2\eta^2}}$ such that the hypothesis $h_t = u\left(\ip{\wh{\bthet}(t)}{\bx}\right)$ satisfies
\begin{align*}
    \zeta\left(h_t\right) &\leq \sqrt{\frac{8L^2 \eta^2 \psi(\bthet)}{\sigma}} + 4\mathcal{R}_n\left(\mathcal{F}\right) + \sqrt{\frac{8\log(1/\delta)}{n}} \:,
\end{align*}
with probability at least $1-\delta$, where $\wh{\bthet}(0) = \argmin_{\bw \in \calM}\psi(\bw)$, and 
\begin{equation*}
    \mathcal{F} = \left\{\bx \mapsto \ip{\bw}{\bx} : \bw \in \calM, \: \bregd{\bthet}{\bw} \leq \psi(\bthet)\right\}.
\end{equation*}
\end{thm}
\begin{proof}
Consider the rate of change of the Bregman divergence between the parameters for the Bayes-optimal predictor $\bthet$ and the parameter estimates $\wh{\bthet}(t)$, \begin{equation*}
    \frac{d}{dt}\bregd{\bthet}{\wh{\bthet}} = \ip{\wh{\bthet}-\bthet}{ \nabla^2\psi\left(\wh{\bthet}\right)\dot{\wh{\bthet}}}.
\end{equation*} 
Observe that $\frac{d}{dt}\nabla\psi\left(\wh{\bthet}\right) = \nabla^2\psi\left(\wh{\bthet}\right)\dot{\wh{\bthet}}$, so that
\begin{align*}
    &\frac{d}{dt}\bregd{\bthet}{\wh{\bthet}} = \frac{1}{n}\sum_{i=1}^n\left(y_i - u\left(\ip{\bx_i}{ \bthet}\right)\right)\ip{\bx_i}{\wh{\bthet} - \bthet} + \frac{1}{n}\sum_{i=1}^n\left(u\left(\ip{\bx_i}{\bthet}\right) - u\left(\ip{\bx_i}{\wh{\bthet}}\right)\right)\ip{\bx_i}{\wh{\bthet} - \bthet} \:.
\end{align*}
Using that $u$ is $L$-Lipschitz and nondecreasing, we may upper bound the second term by $-\frac{1}{L}\wh{\varepsilon}(h_t)$,
\begin{align}
    \frac{d}{dt}\bregd{\bthet}{\wh{\bthet}} \leq \frac{1}{n}\sum_{i=1}^n\left(y_i - u\left(\ip{\bx_i}{\bthet}\right)\right)\ip{\bx_i}{\wh{\bthet} - \bthet} - \frac{1}{L}\wh{\varepsilon}(h_t) \:.
    \label{eqn:breg_dt}
\end{align}
By assumption, $\Vert\frac{1}{n}\sum_{i=1}^n\left(y_i - u(\left\langle\bx_i, \bthet\right\rangle)\right)\bx_i\Vert_* \leq \eta$. Now, observe that by strong convexity of $\psi$ and by the initialization,
\begin{equation*}
    \norm{\wh{\bthet}(0) - \bthet} \leq \sqrt{\frac{2\bregd{\bthet}{\wh{\bthet}(0)}}{\sigma}} \leq \sqrt{\frac{2\psi(\bthet)}{\sigma}} \:.
\end{equation*}
By induction, assume that $\bregd{\bthet}{\wh{\bthet}(t)} \leq \psi(\bthet)$ at time $t$. Then we have the bound
\begin{equation*}
    \frac{d}{dt}\bregd{\bthet}{\wh{\bthet}} \leq -\frac{1}{L}\wh{\varepsilon}(h_t) + \eta\sqrt{\frac{2\psi(\bthet)}{\sigma}} \:,
\end{equation*}
so that either $\frac{d}{dt}\bregd{\bthet}{\wh{\bthet}} < -\eta\sqrt{\frac{2\psi(\bthet)}{\sigma}}$ or $\wh{\varepsilon}(h_t) \leq 2L\eta\sqrt{\frac{2\psi(\bthet)}{\sigma}}$. In the latter case, we have obtained the desired bound on $\wh{\varepsilon}(h_t)$. Otherwise, $t$ cannot exceed 
\begin{equation*}
    t_f = \frac{\bregd{\bthet}{\wh{\bthet}(0)}}{\sqrt{\frac{2\psi(\bthet)}{\sigma}}\eta} = \sqrt{\frac{\psi(\bthet)\sigma}{2\eta^2}}
\end{equation*}
to satisfy $\wh{\varepsilon}(h_t) \leq 2L\eta\sqrt{\frac{2\psi(\bthet)}{\sigma}}$. Hence there is some $h_t$ with $t < t_f$ such that $\wh{\varepsilon}(h_t) \leq 2L\eta\sqrt{\frac{2\psi(\bthet)}{\sigma}}$.
To transfer this bound on $\wh{\varepsilon}$ to $\varepsilon$, we need to bound $\left|\wh{\varepsilon}(h_t) - \varepsilon(h_t)\right|$. Application of a standard uniform convergence
result (cf.\ Theorem~\ref{thm:ull_loss}) to the square loss\footnote{Note that while the square loss is neither bounded nor Lipschitz in general, it is both over the domain $[0, 1]$ with bound $b = 1$ and Lipschitz constant $L' = 1$.} implies
\begin{equation*}
    \left|\wh{\varepsilon}(h_t) - \varepsilon(h_t)\right| \leq 4\mathcal{R}_n\left(\mathcal{F}\right) + \sqrt{\frac{8\log(1/\delta)}{n}}
\end{equation*}
with probability at least $1-\delta$.
\end{proof}
Because $\varepsilon(h_t)$ = $\err(h_t)$ up to a constant, we can find a good predictor by using a hold-out set to estimate $\err(h_t)$ throughout learning.

The statement of Theorem~\ref{thm:stat} uses a specific initialization strategy to write the generalization error bound in terms of $\psi(\bthet)$; with an arbitrary initialization, $\psi(\bthet)$ can be replaced by $\bregd{\bthet}{\wh{\bthet}(0)}$, and our definition of $\calF$ can be modified accordingly. As the bound depends on $\psi(\bthet)$, $C$, and $\eta$, the potential $\psi$ may be chosen in correspondence with available knowledge on the problem structure to optimize the guarantee on the generalization error. In Corollaries~\ref{cor:pq_1}-\ref{cor:ent}, we provide explicit illustrations of this fact. In the experiments in Section~\ref{sec:experiments}, we show how this can be used for improved estimation over the GLM-tron in problems such as sparse vector and low-rank matrix recovery.

Our proof of Theorem~\ref{thm:stat} is similar to the corresponding proof for the GLM-tron~\citep{glmtron}, but has two primary modifications. First, we consider the Bregman divergence under $\psi$ between the Bayes-optimal parameters and the current parameter estimates, rather than the squared Euclidean distance. Our use of Bregman divergence critically relies on the Bayes-optimal parameters appearing in the first argument. Second, rather than analyzing the \textit{iteration} on $\Vert \wh{\bthet}_t - \bthet\Vert_2^2$ as in the discrete-time case, we analyze the \textit{dynamics} of the Bregman divergence. Taking $\psi = \frac{1}{2}\Vert\cdot\Vert_2^2$ recovers the guarantee of the GLM-tron up to forward Euler discretization-specific details.
\subsection{Implicit regularization}
We now study how the choice of $\psi$ impacts the model learned by~\eqref{eqn:ref}. To do so, we require a realizability assumption on the dataset.
\begin{assumption}
\label{assmp:real}
There exists a fixed $\bthet\in\mathbb{R}^d$ such that $y_i = u\left(\left\langle\bthet, \bx_i\right\rangle\right)$ for all $i = 1, \hdots,n$.
\end{assumption}
In many cases, even the noisy dataset of Section~\ref{sec:stat} may satisfy Assumption~\ref{assmp:real} for a $\bar{\bthet} \neq \bthet$. We now begin by proving convergence of the training error.
\begin{lem}
\label{lem:conv}
Suppose that $\{\bx_i, y_i\}_{i=1}^n$ are drawn i.i.d. from a distribution $\mathcal{D}$ supported on $\mathcal{X}\times [0, 1]$. Let the dataset satisfy Assumption~\ref{assmp:real}, let $u$ satisfy Assumption~\ref{assmp:link}, and let $\psi$ satisfy Assumption~\ref{assmp:sc}. Suppose $\norm{\bx_i}_* \leq C$. Then $\wh{\varepsilon}(h_t) \rightarrow 0$ where $h_t(\bx) = u\left(\left\langle\wh{\bthet}(t), \bx\right\rangle\right)$ and $\wh{\bthet}(0) = \argmin_{\bw\in\calM} \psi(\bw)$. Furthermore, $\min_{t'\in [0, t]}\left\{\wh{\varepsilon}(h_{t'})\right\} \leq \mathcal{O}\left(1/t\right)$.
\end{lem}
\begin{proof}
Under the assumptions, (\ref{eqn:breg_dt}) implies
\begin{align*}
    \frac{d}{dt}\bregd{\bthet}{\wh{\bthet}} &\leq -\frac{1}{L}\wh{\varepsilon}(h_t) \leq 0 \:.
\end{align*}
Integrating both sides of the above gives the bound
\begin{align*}
    \int_0^t\wh{\varepsilon}(h_{t'})dt' &\leq L\bregd{\bthet}{\wh{\bthet}(0)}.
\end{align*}
Explicit computation shows that $\frac{d}{dt}\wh{\varepsilon}(h_t)$ is bounded, so that $\wh{\varepsilon}(h_t)$ is uniformly continuous in $t$. By Barbalat's Lemma (cf. Lemma~\ref{lem:barbalat}), this implies that $\wh{\varepsilon}\rightarrow 0$ as $t\rightarrow\infty$. Now, note that
\begin{align*}
    \inf_{t'\in [0, t]}\left\{\wh{\varepsilon}(h_{t'})\right\}t &= \int_0^t\inf_{t'\in [0, t]}\left\{\wh{\varepsilon}(h_{t'})\right\}dt''\\
    &\leq \int_0^t\wh{\varepsilon}(h_{t'})dt' \leq L\bregd{\bthet}{\wh{\bthet}(0)},
\end{align*}
so that $\inf_{t'\in [0, t]}\left\{\wh{\varepsilon}(h_{t'})\right\} \leq \frac{L\bregd{\bthet}{\wh{\bthet}(0)}}{t}$.
\end{proof}
Lemma~\ref{lem:conv} shows that \eqref{eqn:ref} will converge to an interpolating solution for a realizable dataset, and that the best hypothesis up to time $t$ does so at an $\mathcal{O}\left(1/t\right)$ rate; the proof is given in the appendix.

In general, there may be many possible vectors $\wh{\bthet}$ consistent with the data. The following theorem provides insight into the parameters learned by \eqref{eqn:ref}. Our result is analogous to the characterization of the implicit bias of mirror descent due to~\citet{gunasekar_3}, and uses a continuous-time proof technique inspired by the discrete-time technique in~\citet{azizan_1}. A similar continuous-time proof first appeared in~\citet{boffi2019higherorder} in the context of adaptive control.
\begin{thm}
\label{thm:imp_reg}
Consider the setting of Lemma~\ref{lem:conv}. Let $\mathcal{A} = \{\bar{\bthet}\in\calM: u\left(\left\langle\bar{\bthet}, \bx_i\right\rangle\right) = y_i, i = 1, \hdots,n\}$ be the set of parameters that interpolate the data, and assume that $\wh{\bthet}(t)\rightarrow\wh{\bthet}_{\infty} \in \mathcal{A}$. Further assume that $u(\cdot)$ is invertible. Then $\wh{\bthet}_{\infty} = \argmin_{\bar{\bthet} \in \mathcal{A}}\bregd{\bar{\bthet}}{\wh{\bthet}(0)}$. In particular, if $\wh{\bthet}(0) = \argmin_{\bw\in\calM}\psi(\bw)$, then $\wh{\bthet}_\infty = \argmin_{\bar{\bthet}\in\mathcal{A}}\psi(\bar{\bthet})$.
\end{thm}
\begin{proof}
Let $\bar{\bthet} \in \mathcal{A}$ be arbitrary. Define the error on example $i$ as $\tilde{y}_i\left(\wh{\bthet}(t)\right) = \left(u\left(\left\langle\wh{\bthet}(t), \bx_i\right\rangle\right) - y_i\right)$. Then,
\begin{align*}
    \frac{d}{dt}\bregd{\bar{\bthet}}{\wh{\bthet}(t)} &= -\frac{1}{n}\sum_{i=1}^n\tilde{y}_i\left(\wh{\bthet}(t)\right)\left\langle\wh{\bthet}(t)-\bar{\bthet}, \bx_i\right\rangle,\\
    &= -\frac{1}{n}\sum_{i=1}^n\tilde{y}_i\left(\wh{\bthet}(t)\right)\left(\ip{\wh{\bthet}(t)}{\bx_i} - u^{-1}\left(y_i\right)\right).
\end{align*}
Above, we used that $\bar{\bthet} \in \mathcal{A}$ and that $u(\cdot)$ is invertible, so that $u\left(\left\langle\bar{\bthet}, \bx_i\right\rangle\right) = y_i$ implies that $\left\langle\bar{\bthet}, \bx_i\right\rangle = u^{-1}(y_i)$. Integrating both sides of the above, we find that 
\begin{align*}
    \bregd{\bar{\bthet}}{\wh{\bthet}_\infty} - \bregd{\bar{\bthet}}{\wh{\bthet}(0)} = -\frac{1}{n}\sum_{i=1}^n\int_0^\infty\tilde{y}_i\left(\wh{\bthet}(t)\right)\left(\left\langle\wh{\bthet}(t), \bx_i\right\rangle - u^{-1}\left(y_i\right)\right)dt.
\end{align*}
The above relation is true for any $\bar{\bthet} \in \mathcal{A}$. Furthermore, the integral on the right-hand side is independent of $\bar{\bthet}$. Hence the $\argmin$ of the two Bregman divergences must be equal, which shows that $\wh{\bthet}_{\infty} = \argmin_{\bar{\bthet}\in\mathcal{A}}\bregd{\bar{\bthet}}{\wh{\bthet}(0)}$.
\end{proof}
Theorem~\ref{thm:imp_reg} elucidates the implicit bias of pseudogradient algorithms captured by \eqref{eqn:ref}. Out of all possible interpolating parameters, \eqref{eqn:ref} finds those that minimize the Bregman divergence between the set of interpolating parameters and the initialization.

%% file: disc.tex
\label{sec:disc}%
Equation \eqref{eqn:ref} can be discretized via Forward-Euler to form an implementation with a step size $\lambda > 0$,
\begin{align}
    \label{eqn:refl_disc}
    \nabla\psi\left(\wh{\bphi}_{t+1}\right) &= \nabla\psi\left(\wh{\bthet}_t\right) - \frac{\lambda}{n}\sum_{i=1}^n\left(u\left(\ip{\wh{\bthet}_t}{\bx_i}\right) - y_i\right)\xi\left(\wh{\bthet}, \bx_i\right)\bx_i,\\
    \wh{\bthet}_{t+1} &= \Pi_{\calC}^\psi\left(\wh{\bphi}_{t+1}\right).
    \label{eqn:refl_disc_proj}
\end{align}
In \eqref{eqn:refl_disc_proj}, $\calC$ denotes a convex constraint set and $\Pi_{\calC}^\psi(\bz) = \argmin_{\bx \in \calC\cap\calM}\bregd{\bx}{\bz}$ denotes the Bregman projection. To analyze the iteration \eqref{eqn:refl_disc}~\&~\eqref{eqn:refl_disc_proj}, we need two assumptions on $\xi$.
\begin{assumption}[Adapted from~\citet{agnostic_neuron}]
\label{assmp:frei}
For any $a > 0$ and $b > 0$, there exists a $\gamma > 0$ such that $\inf_{\norm{\bw}\leq a, \norm{\bx}_* \leq b}\xi\left(\bw, \bx\right) \geq \gamma > 0$.
\end{assumption}
For mirror descent, Assumption~\ref{assmp:frei} reduces to a requirement that the derivative of the activation remains nonzero over any compact set.
\begin{assumption}
\label{assmp:xi_bound}
There exists a constant $B > 0$ such that $\xi\left(\bw, \bx\right) \leq B$ for all $\bw \in \calM$, $\bx\in\calX$.
\end{assumption}
For mirror descent, we take $B = L$, while for the mirror descent generalization of GLM-tron, we take $B=1$. We may now state our statistical guarantees.

\begin{thm}
\label{thm:stat_disc}
Suppose that $\{\bx_i, y_i\}_{i=1}^n$ are drawn i.i.d. from a distribution $\mathcal{D}$ supported on $\mathcal{X}\times[0, 1]$ where $\mathbb{E}\left[y|\bx\right] = u\left(\ip{\bthet}{\bx}\right)$, $u$ satisfies Assumptions~\ref{assmp:link}~\&~\ref{assmp:frei}, and $\bthet\in\calC$ is an unknown vector of parameters. Let $\psi$ satisfy Assumption~\ref{assmp:link}, and let $\xi$ satisfy Assumptions~\ref{assmp:frei}~\&~\ref{assmp:xi_bound}. Assume that $\norm{\bx_i}_* \leq C$, $\norm{\bthet} \leq W$, and $\norm{\frac{1}{n} \sum_{i=1}^n \left(u\left(\ip{\bthet}{\bx_i}\right) - y_i\right)\xi\left(\ip{\bthet}{\bx_i}\right)\bx_i}_* \leq \eta$. Let $\gamma$ correspond to $a=C$ and $b = W + \sqrt{\frac{2\psi(\bthet)}{\sigma}}$ in Assumption~\ref{assmp:frei}. Then for $\lambda \leq \frac{\sigma}{2C^2BL}$ there exists some iteration $t < \frac{1}{\lambda}\sqrt{\frac{\sigma\psi(\bthet)}{2\eta^2}}$ such that $h_t = u\left(\ip{\wh{\bthet}_t}{\bx}\right)$ satisfies with probability at least $1-\delta$
\begin{align*}
     \zeta(h_t) &\leq \sqrt{\frac{32L^2\eta^2\psi(\bthet)}{\gamma^2\sigma}}\left(\frac{2C^2LB + 1}{2C^2LB}\right) + 4\mathcal{R}_n\left(\mathcal{F}\right) + \sqrt{\frac{8\log(1/\delta)}{n}},
\end{align*}
where $\wh{\bthet}_1 = \argmin_{\bw\in\calC\cap\calM}\psi(\bw)$, and where $\mathcal{F} = \left\{\bx \mapsto \ip{\bw}{\bx} : \bw \in \calM, \: \bregd{\bthet}{\bw} \leq \psi(\bthet)\right\}$.
\end{thm}
Theorem~\ref{thm:stat_disc} shows that, for a suitable choice of step size, the discrete-time iteration \eqref{eqn:refl_disc}~\&~\eqref{eqn:refl_disc_proj} preserves the guarantees of the continuous-time flow \eqref{eqn:ref}.
The proof (and all subsequent proofs) are given in the appendix. We now state several consequences of Theorem~\ref{thm:stat_disc} in standard settings that highlight the impact of the potential on generalization in nonconvex learning.
\begin{cor}[$p/q$ dual norm pairs with $p \in [2, \infty)$]
\label{cor:pq_1}
Let $\norm{\cdot}_* = \norm{\cdot}_p$ with $p\in [2, \infty)$. Then $\psi(\bw) = \frac{1}{2}\norm{\bw}_q^2$ is $(q-1)$-strongly convex with respect to $\norm{\cdot}_q$ where $1/q + 1/p = 1$. The generalization error is bounded as
\begin{align*}
    \zeta(h_t) &\leq \frac{4LWC}{q-1}\left(\frac{\sqrt{2\log(4/\delta)(q-1)} + 1}{\sqrt{n}}\right)\frac{2C^2LB+1}{2C^2LB}+ \frac{4CW}{\sqrt{n(q-1)}}\left(1 + \frac{1}{\sqrt{q-1}}\right) + \sqrt{\frac{8\log(1/\delta)}{n}} \:.
\end{align*}
\end{cor}
\begin{cor}[$\infty/1$ dual norm pairs, global setup]
\label{cor:pq_2}
Let $\norm{\cdot} = \norm{\cdot}_1$ and $\norm{\cdot}_* = \norm{\cdot}_\infty$. Then $\psi(\bw) = \frac{1}{2}\norm{\bw}_q^2$ with $q = \frac{\log(d)}{\log(d)-1}$ is $\frac{1}{3\log(d)}$-strongly convex with respect to $\norm{\cdot}_1$. Then, the generalization error can be bounded
\begin{align*}
    &\zeta(h_t) \leq \frac{4CW(1+\sqrt{3\log d})^2}{n^{1/2}} + \sqrt{\frac{8\log(1/\delta)}{n}} + \frac{12LCW\sqrt{3\log(d)}(2C^2LB+1)}{C^2LB}\sqrt{\frac{\log(4d/\delta)}{n}} \:.
\end{align*}
\end{cor}
\begin{cor}[$\infty/1$ dual norm pairs, simplex setup]
\label{cor:ent}
Let $\norm{\cdot} = \norm{\cdot}_1$ and $\norm{\cdot}_* = \norm{\cdot}_\infty$. Take $\psi(\bw) = \dkl{\bw}{\bu}$ where $\bu$ is the discrete uniform distribution in $d$ dimensions and where $d_{KL}$ denotes the KL divergence. Then $\psi(\bw)$ is $1$-strongly convex with respect to $\norm{\cdot}_1$ over the probability simplex and $\psi(\bw) \leq \log(d)$ for any $\bw$. Then,
\begin{align*}
    &\zeta(h_t) \leq 4C\sqrt{\frac{2\log d}{n}} + \sqrt{\frac{8\log(1/\delta)}{n}} + \frac{3LC \sqrt{32\log d}(2C^2LB+1)}{C^2LB}\sqrt{\frac{\log(4d/\delta)}{n}} \:.
\end{align*}
\end{cor}
In the above results, the dimensionality dependence of the generalization error has been reduced by judicious choice of $\psi$. In particular, Corollaries~\ref{cor:pq_2} and~\ref{cor:ent} are merely logarithmic in dimension, while a bound for the GLM-tron would be polynomial in dimension.

Similar to Theorem~\ref{thm:stat_disc}, we now show that Lemma~\ref{lem:conv} and Theorem~\ref{thm:imp_reg} are preserved when discretizing \eqref{eqn:ref}. We first state a convergence guarantee.
\begin{lem}
\label{lem:conv_disc}
Suppose that $\{\bx_i, y_i\}_{i=1}^n$ are drawn i.i.d.\ from a distribution $\mathcal{D}$ supported on $\mathcal{X}\times [0, 1]$. Let the dataset satisfy Assumption~\ref{assmp:real} let $u$ satisfy Assumption~\ref{assmp:link}, and let $\psi$ satisfy Assumption~\ref{assmp:sc}. Suppose $\norm{\bx_i}_* \leq C$. Then for $\lambda \leq \frac{2\sigma}{C^2BL}$, $\wh{\varepsilon}(h_t) \rightarrow 0$ where $h_t(\bx) = u\left(\left\langle\wh{\bthet}(t), \bx\right\rangle\right)$ is the hypothesis with parameters output by \eqref{eqn:refl_disc}~\&~\eqref{eqn:refl_disc_proj} at time $t$ with $\wh{\bthet}_1 = \argmin_{\bw\in\calC\cap\calM} \psi(\bw)$. Furthermore, $\min_{t'\in [0, t]}\left\{\wh{\varepsilon}(h_{t'})\right\} \leq \mathcal{O}\left(1/t\right)$.
\end{lem}
We conclude by showing that the implicit bias properties of \eqref{eqn:refl_disc}~\&~\eqref{eqn:refl_disc_proj} match those of \eqref{eqn:ref}.
\begin{thm}
\label{thm:imp_reg_disc}
Consider the setting of Lemma~\ref{lem:conv_disc}, and assume that $u(\cdot)$ is invertible. Let $\mathcal{A} = \{\bar{\bthet} \in \calC\cap\calM: u\left(\left\langle\bar{\bthet}, \bx_i\right\rangle\right) = y_i, i = 1, \hdots,n\}$ be the set of parameters that interpolate the data, and assume that $\wh{\bthet}_t\rightarrow\wh{\bthet}_{\infty} \in \mathcal{A}$. Then $\wh{\bthet}_{\infty} = \argmin_{\bar{\bthet} \in \mathcal{A}}\bregd{\bar{\bthet}}{\wh{\bthet}_1}$. In particular, if $\wh{\bthet}_1 = \argmin_{\bw\in\calC\cap\calM}\psi(\bw)$, then $\wh{\bthet}_\infty = \argmin_{\bar{\bthet}\in\mathcal{A}}\psi(\bar{\bthet})$.
\end{thm}
Taken together, the results in this section show that the continuous guarantees are preserved by discretization, though the analysis requires care of higher-order terms that vanish in the continuous limit.

%% file: online.tex
\label{sec:online}%
In this section, we provide guarantees for the iteration
\begin{align}
    \label{eqn:refl_online}
    \nabla\psi\left(\wh{\bphi}_{t+1}\right) &= \nabla\psi\left(\wh{\bthet}_t\right) - \lambda \left(u\left(\ip{\wh{\bthet}_t}{\bx_t}\right) - y_t\right)\xi\left(\wh{\bthet}_t, \bx_t\right)\bx_t,\\
    \wh{\bthet}_{t+1} &= \Pi_{\calC}^\psi\left(\wh{\bphi}_{t+1}\right),
    \label{eqn:refl_online_proj}
\end{align}
which is similar to stochastic gradient descent. We first consider the bounded noise setting, where we conclude a $\mathcal{O}(1/\sqrt{t})$ convergence rate of the generalization error.
\begin{thm}
\label{thm:online_noise}
Suppose that $\{\bx_t, y_t\}_{t=1}^\infty$ are drawn i.i.d. from a distribution $\mathcal{D}$ supported on $\mathcal{X}\times[0, 1]$ where $\mathbb{E}\left[y|\bx\right] = u\left(\left\langle\bthet, \bx\right\rangle\right)$, $\bthet \in \calC$ is an unknown vector of parameters, and $u$ satisfies Assumption~\ref{assmp:link}. Assume that $\calC$ is compact, and let $R = \diam(\calC)$ as measured in the norm $\norm{\cdot}$. Suppose that $\psi$ satisfies Assumption~\ref{assmp:sc}, and that $\norm{\bx_t}_* \leq C$ for all $t$. Fix a horizon $T$, and choose $\lambda < \min\left\{\frac{2\sigma}{C^2LB}, \frac{1}{\sqrt{T}}\right\}$. Then for any $\delta \in (0, 1)$, with probability at least $1 - \delta$,
\begin{align*}
    \min_{t < T}\varepsilon(h_t) &\leq \mathcal{O}\left(\frac{L}{\sqrt{T}\gamma}\left(\psi(\bthet)+ \sqrt{CBR\log(6/\delta)} + \frac{C^2B^2}{\sigma}\right)\right)
\end{align*}
where $h_t$ is the hypothesis output by \eqref{eqn:refl_online}~\&~\eqref{eqn:refl_online_proj} at iteration $t$ with $\wh{\bthet}_1 = \argmin_{\bthet\in\calC}\psi(\bthet)$, and $\gamma$ corresponds to $a = C$ and $b = R + \norm{\bthet}$ in Assumption~\ref{assmp:frei}.
\end{thm}
We now consider the realizable setting, where we obtain fast $\mathcal{O}(1/t)$ rates.
\begin{thm}
\label{thm:online_real}
Suppose that $\{\bx_t, y_t\}_{t=1}^{\infty}$ are drawn i.i.d. from a distribution $\mathcal{D}$ supported on $\mathcal{X}$. Let Assumption~\ref{assmp:real} be satisfied with $\bthet\in\calC$ an unknown vector of parameters, let $u$ satisfy Assumption~\ref{assmp:link}, let $\psi$ satisfy Assumption~\ref{assmp:sc}, and assume that $\norm{\bx_t}_* \leq C$ for all $t$. Fix $\lambda < \frac{2\sigma}{LC^2B}$. Then for any $\delta \in (0, 1)$, for all $T \geq 1$, with probability at least $1 - \delta$
\begin{equation*}
    \min_{t < T} \varepsilon(h_t) \leq \mathcal{O}\left(\frac{L^2 C^2B \psi(\bthet)\log(1/\delta)}{\sigma T\gamma}\right),
\end{equation*}
where $h_t$ is the hypothesis output by \eqref{eqn:refl_online}~\&~\eqref{eqn:refl_online_proj} at iteration $t$ and where $\gamma$ corresponds to $a = C$ and $b = \norm{\bthet} + \sqrt{\frac{2\psi(\bthet)}{\sigma}}$ in Assumption~\ref{assmp:frei}.
\end{thm}

%% file: sims.tex
\label{sec:experiments}
\begin{figure*}[ht]
     \centering
     \begin{subfigure}[b]{0.3275\textwidth}
         \centering
         \begin{overpic}[width=\textwidth]{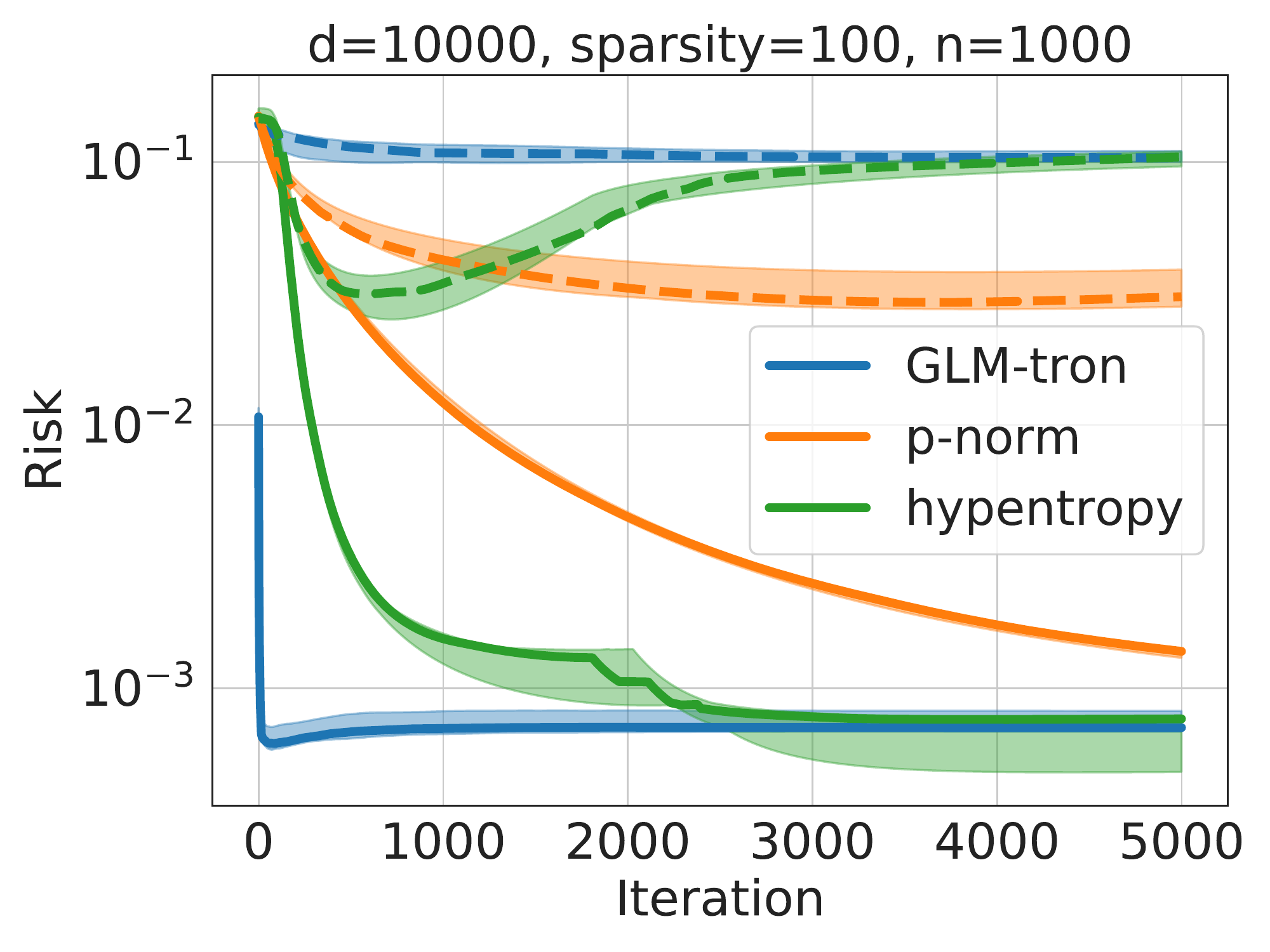}
         \put(5, 70){\textbf{a}}
         \end{overpic}
         \caption{\label{fig:sparse_vector_training_curve}}
     \end{subfigure}
     \hfill
     \begin{subfigure}[b]{0.3275\textwidth}
         \centering
         \begin{overpic}[width=\textwidth]{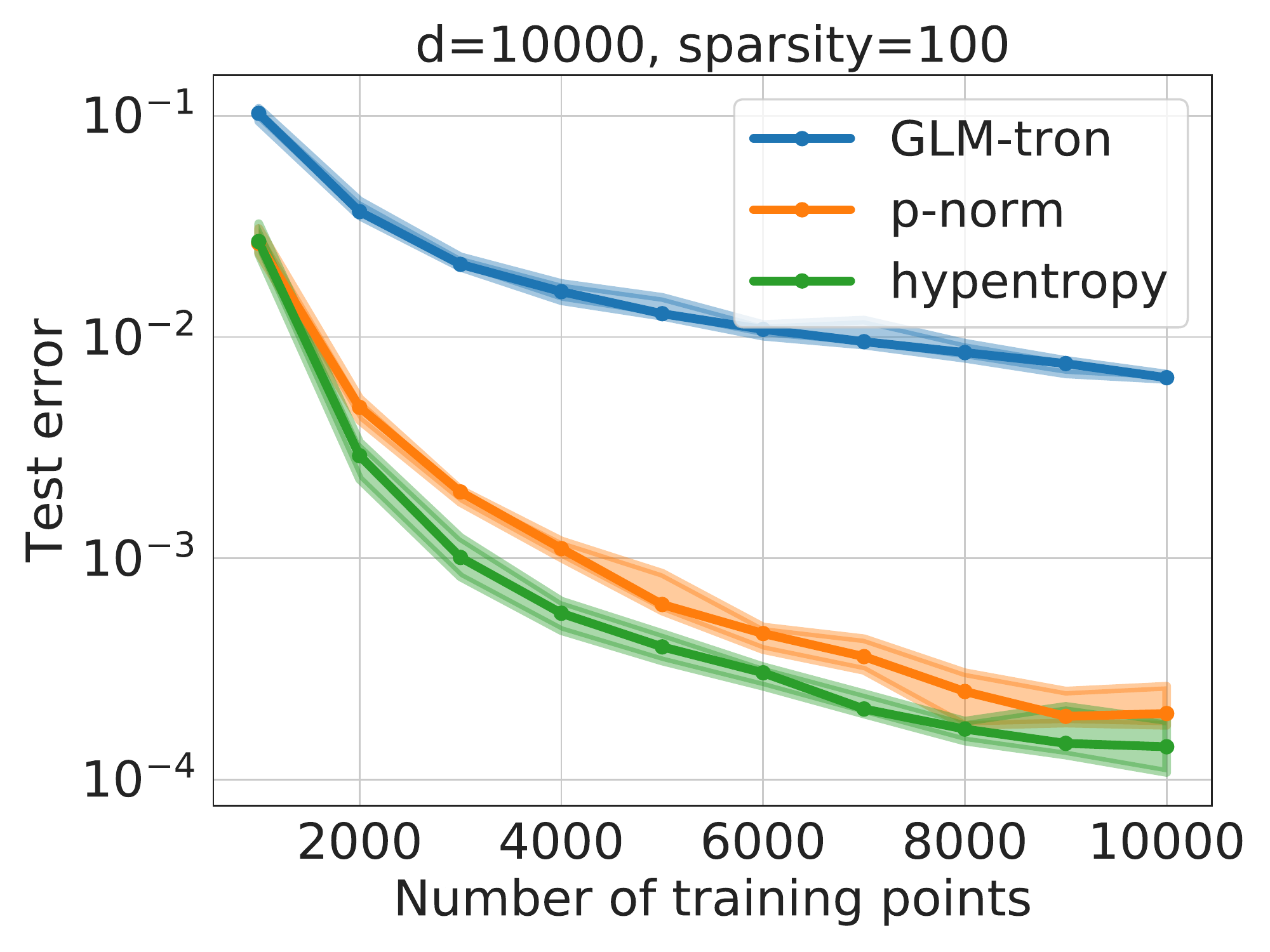}
         \put(5, 70){\textbf{b}}
         \end{overpic}
         \caption{\label{fig:sparse_vector_fixed_d}}
     \end{subfigure}
     \hfill
     \begin{subfigure}[b]{0.3275\textwidth}
         \centering
         \begin{overpic}[width=\textwidth]{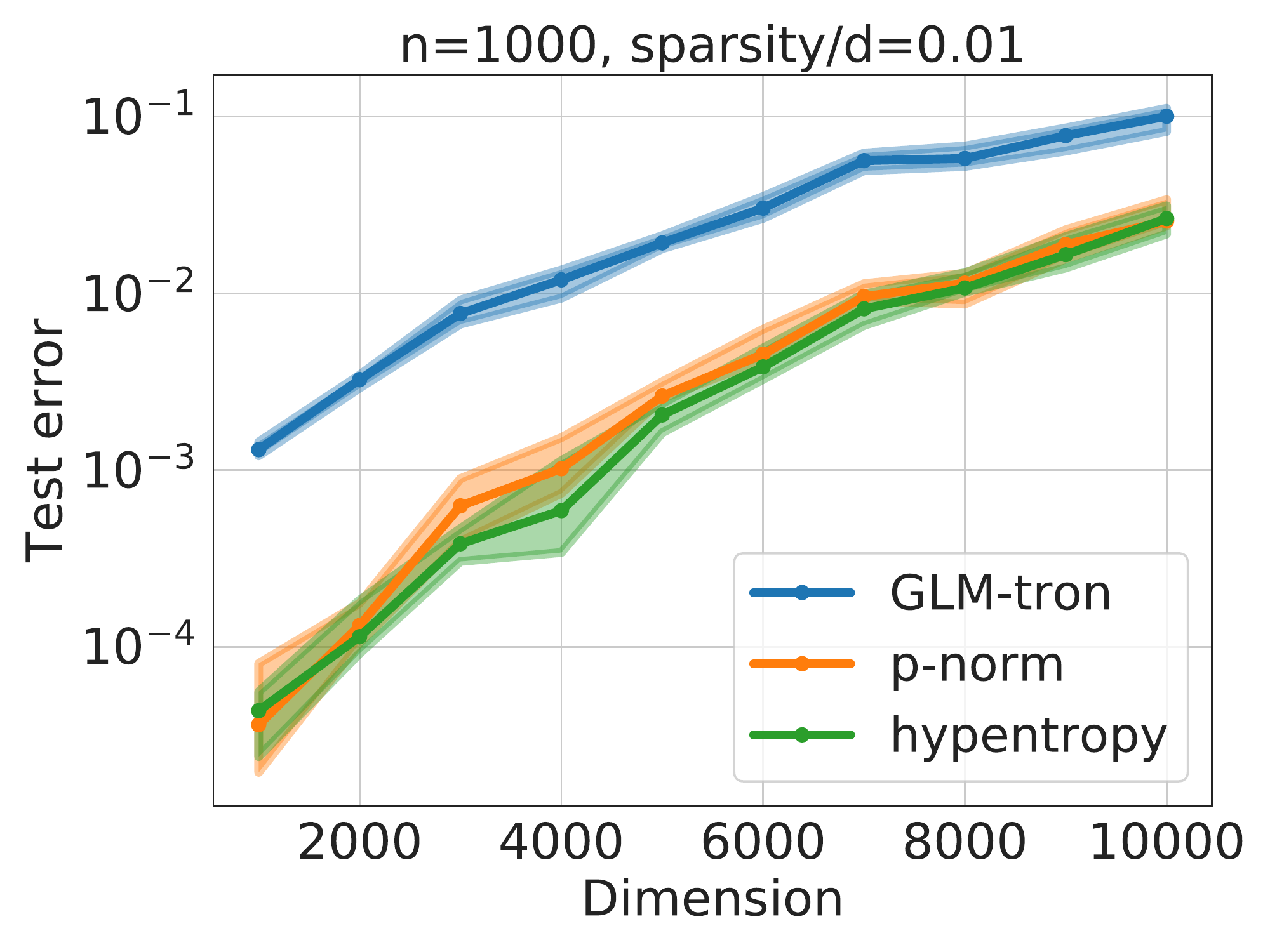}
         \put(5, 70){\textbf{c}}
         \end{overpic}
         \caption{\label{fig:sparse_vector_fixed_n}}
     \end{subfigure}
     \begin{subfigure}[b]{0.3275\textwidth}
         \centering
         \begin{overpic}[width=\textwidth]{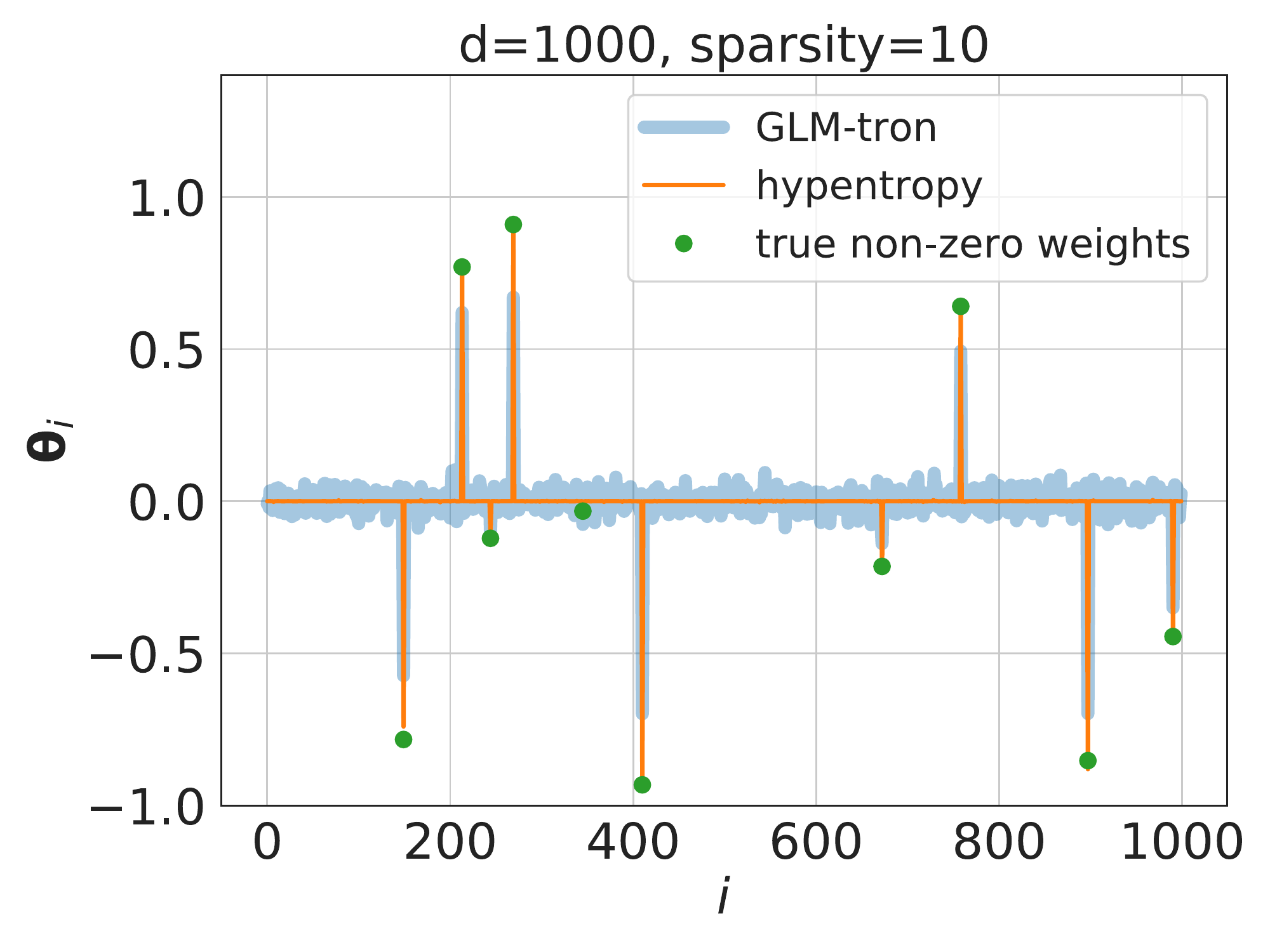}
         \put(5, 70){\textbf{d}}
         \end{overpic}
         \caption{\label{fig:sparse_vector_weights}}
     \end{subfigure}
     \hfill
     \begin{subfigure}[b]{0.3275\textwidth}
         \centering
         \begin{overpic}[width=\textwidth]{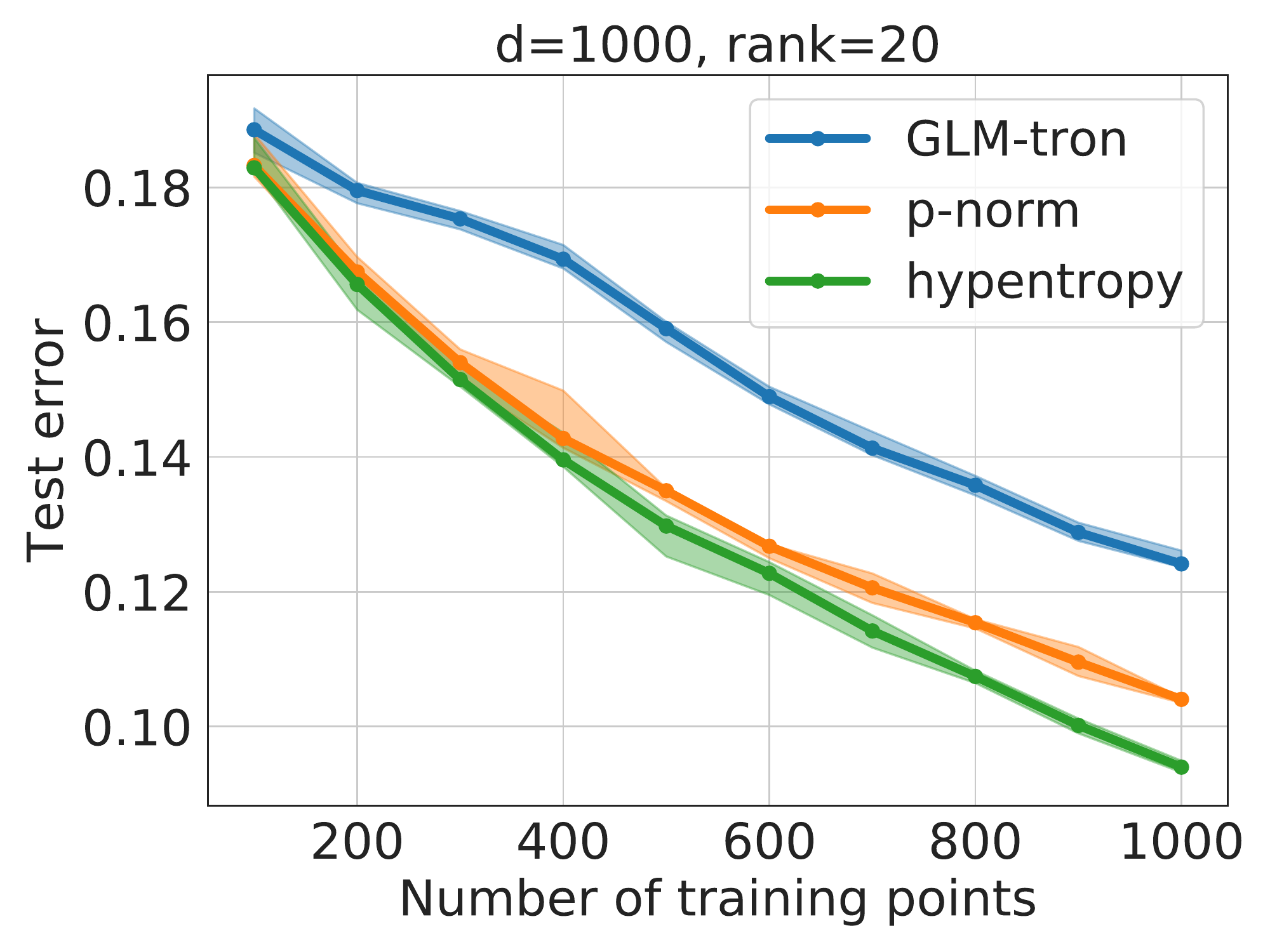}
         \put(5, 70){\textbf{e}}
         \end{overpic}
         \caption{\label{fig:low_rank_fixed_d}}
     \end{subfigure}
     \hfill
     \begin{subfigure}[b]{0.3275\textwidth}
         \centering
         \begin{overpic}[width=\textwidth]{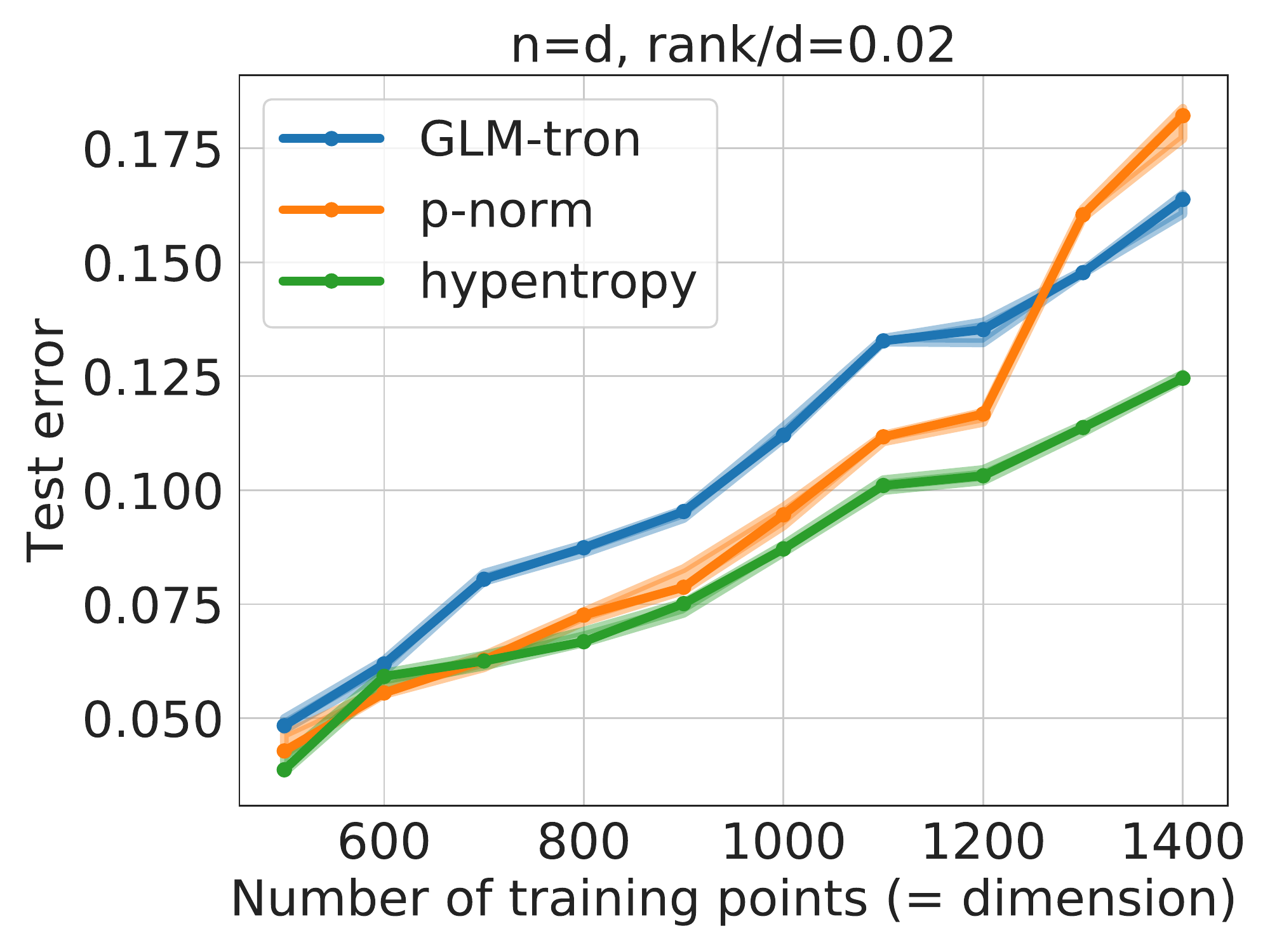}
         \put(5, 70){\textbf{f}}
         \end{overpic}
         \caption{\label{fig:low_rank_fixed_ratio}}
     \end{subfigure}
    \caption{(a) Risk curves. Solid indicates training error and dashed indicates holdout error. (b) Test error with $d$ and $s$ fixed as $n$ varies. (c) Test error with $n$ and $s/d$ fixed as $d$ varies. (d) Weights learned without explicit projection. (e) Test error with $d$ and $r$ fixed as $n$ varies. (f) Test error with fixed $r/d$ as $n=d$ varies.}
    \label{fig:main_fig}
\end{figure*}
We now illustrate our theoretical results in two concrete problem settings. We first study a scalar-valued output problem where the Bayes-optimal parameter vector is sparse. We then consider a vector-valued system identification problem where the Bayes-optimal parameter matrix is low-rank.

We compare three variants of the Reflectron with three different choices of potential. The first is the GLM-tron,
which is equivalent to the use of the Euclidean
potential $\psi_2(\bx) = \frac{1}{2} \norm{\bx}^2_2$.
The second is the $p$-norm algorithm~\citep{p_norm},
which uses the potential $\psi_p(\bx) = \frac{1}{2} \norm{\bx}_p^2$
for $p \in [1, \infty]$.
The third variant is the hypentropy 
algorithm~\citep{ghai20hypentropy}, which generalizes the setup considered in Corollary~\ref{cor:ent} beyond the probability simplex and uses the potential $\psi_\beta(\bx) = \sum_{i=1}^{d} (\bx_i \arcsinh(\bx_i/\beta) - \sqrt{\bx_i^2 + \beta^2})$ for $\beta \in (0, \infty)$. 
A complete description of the experimental setup is given in the appendix.

\subsection{Sparse vector GLMs}
\label{sec:exp:sparse_vector}
In this setting, the learner receives measurements 
$y_i = \sigma(\ip{\bthet}{\bx_i}) + w_i$ with $\bx_i \sim \mathrm{Unif}([-1, 1]^d)$, $w_i \sim \mathrm{Unif}([-\sigma_w, \sigma_w]^d)$, and where $\sigma(\cdot)$ is the sigmoid activation. $\bthet$ is taken to be an $s$-sparse vector with $s \ll d$, and we compare the GLM-tron with \textit{explicit} $\ell_1$ projection to the Reflectron with matching \textit{explicit} $\ell_1$ projection and the \textit{implicit} regularization due to the either the hypentropy or $p$-norm potentials.

The learner has knowledge that $\bthet$ is sparse, as well as access to the upper bounds $\norm{\bthet}_p \leq W_p$. In the experiments, we set $W_p = 2\norm{\bthet}_p$. Let $\mathbb{B}_p(r)$ denote the closed $\ell_p$-ball in $\R^d$ of radius $r$ centered at the origin. For the GLM-tron, we set $\calC = \mathbb{B}_1(W_1)$ and the projection onto $\calC$ is Euclidean. For the $p$-norm algorithm, we set $\calC = \mathbb{B}_p(W_p)$ and apply a Bregman projection onto $\calC$. For hypentropy, we set $\calC = \mathbb{B}_1(W_1)$ and again use the Bregman projection onto $\calC$. We compare each algorithm in two experimental regimes.

In the first regime, the ambient dimension $d=10000$ and sparsity $s=100$ are fixed, and we study the performance of each algorithm as a function of the number of data points $n$.
For each pair of $(n, \mathrm{alg})$, we run the full-batch pseudogradient algorithm for
$5000$ iterations
over a grid of hyperparameters, and we tune the step size $\lambda$ and the $p$ value for the $p$-norm algorithm (resp.\ $\beta$ for hypentropy). 
As suggested by Theorem~\ref{thm:stat_disc}, we use a holdout set of size $n_{\mathrm{hold}} = 500$ to select the parameters with lowest validation error over $5000$ iterations. 
Each algorithm is run for $5$ trials and the configuration that achieves the lowest median test error (over the $5$ trials) is shown in the figures. The size of the test set is $n_{\mathrm{test}} = 1000$, and the error bars correspond to the min/median/max over the $5$ trials.

Figure~\ref{fig:sparse_vector_training_curve} shows the training error
and holdout error of the best configuration for each algorithm with 
$n=1000$. Each algorithm overfits, and the holdout set is necessary to find the predictor with lowest generalization error.
Figure~\ref{fig:sparse_vector_fixed_d} shows the resulting test error of each algorithm. For each value of $n$, both the $p$-norm and hypentropy algorithms have lower test error when compared to the GLM-tron, in line with the generalization error predictions of Theorem~\ref{thm:stat_disc}.

In the second regime, the number of data points is fixed at $n=1000$ while the ambient dimension $d$ is varied for fixed $s/d = 0.01$. Figure~\ref{fig:sparse_vector_fixed_n} shows the test error for each algorithm, which increases with the ambient dimension $d$ in all cases. As in Figure~\ref{fig:sparse_vector_fixed_d}, for fixed $d$, both the $p$-norm and hypentropy algorithms have lower test error than the GLM-tron. Taken together, Figures~\ref{fig:main_fig}(a)-(c) validate the claims of Theorem~\ref{thm:stat_disc}.

To verify the predictions of Theorem~\ref{thm:imp_reg_disc}, we remove the explicit projection onto $\calC$ for both the GLM-tron and hypentropy and visualize the structure of the learned parameter vector in Figure~\ref{fig:sparse_vector_weights} ($d = 1000$, $s=10$, and $n=1000$). Figure~\ref{fig:sparse_vector_weights} shows that hypentropy recovers a much sparser solution than the GLM-tron despite the lack of an explicit projection onto the $\ell_1$-ball. In particular, $971$ coordinates have absolute value greater than $0.001$ for the parameters found by the GLM-tron, while there are only $56$ for hypentropy. Moreover, the qualitative structure of the parameter vector found by hypentropy is much closer to that of the true parameters, and quantitatively $\norm{\wh{\bthet}_{\mathrm{glm}} - \bthet}_1 = 24.609$ while $\norm{\wh{\bthet}_{\mathrm{hyp}} - \bthet}_1 = 0.421$. The parameters found by the $p$-norm algorithm have similar structure to those found by hypentropy and are omitted for visual clarity.
\subsection{Low rank system identification}
\label{sec:exp:low_rank}
We now consider a nonlinear system identification problem similar to~\citet{foster2020learning}, where the system dynamics are given by a vector-valued GLM and the parameters may be identified using a spectral variant of the Reflectron. In this setting, the learner observes $n$ trajectories $\{\bx^i_t\}_{t=0, i=1}^{T, n}$ from the discrete-time dynamical system $\bx_{t+1}^i = \rho \bx_t^i + \sigma(\bThet \bx_t^i) + w_t^i$. The system is initialized from $\bx_0^i \sim \mathrm{Unif}([-1,1]^d)$, the process noise is given by $w_t^i \sim \mathrm{Unif}([-\sigma_w, \sigma_w]^d)$, $\sigma(\cdot)$ is the element-wise sigmoid activation, and $\bThet$ is a $d \times d$ matrix with $r = \mathrm{rank}(\bThet) \ll d$. This model is motivated by applications in computational neuroscience, where the system state can be interpreted as a vector of firing rates in a recurrent neural network, and the learned parameters represent the synaptic weights~\citep{rutishauser15}. In the experiments, we fix $\rho = 0.9$, $T=5$, and $\sigma_w = 0.1$.

The generalization error for an estimate $\wh{\bThet}$ is
$\varepsilon(\wh{\bThet}) = \frac{1}{2T} \sum_{t=0}^{T-1} \Expsub{\bx_0}{ \norm{ \bx_{t+1} - \rho \bx_t - \sigma(\wh{\bThet} \bx_t)}^2  }$, which measures the ability of the learned connectivity to correctly predict a new random trajectory in a mean square sense. We search for $\wh{\bThet}$ by minimizing the empirical loss
$\wh{\varepsilon}(\wh{\bThet}) = \frac{1}{2nT} \sum_{i=1, t=0}^{n, T-1} \norm{ \bx_{t+1}^i - \rho \bx_t^i - \sigma(\wh{\bThet} \bx_t^i) }^2$.

Similar to Section~\ref{sec:exp:sparse_vector}, the learner has knowledge that $\bThet$ is low-rank, and we compare how the implicit bias of each method impacts its generalization performance. Let $\lambda(\mathbf{M})$ denote the vector of singular values of a matrix $\mathbf{M}$. For both the GLM-tron and hypentropy algorithms, we project onto 
$\calC = \left\{ \wh{\bThet} \in \R^{d \times d} : \norm{\lambda(\wh{\bThet})}_1 \leq 2 \norm{\lambda(\bThet)}_1 \right\}$. For the $p$-norm algorithm, we project onto $\calC = \left\{ \wh{\bThet} \in \R^{d \times d} : \norm{\lambda(\wh{\bThet})}_p \leq 2 \norm{\lambda(\bThet)}_p \right\}$. Hyperparameters are tuned just as in the sparse vector setting.

In Figure~\ref{fig:low_rank_fixed_d}, the ambient dimension
and rank are fixed to be $d=1000$ and $r=20$, and we study the impact of the number of trajectories $n$ on the generalization error. Both the $p$-norm and hypentropy algorithms achieve lower test error than the GLM-tron algorithm for fixed $n$. In Figure~\ref{fig:low_rank_fixed_ratio}, the ambient dimension $d$ is varied with fixed $r/d = 0.02$, and the number of trajectories is held equal to the dimension $n=d$. A heuristic explanation of this scaling is provided in Appendix~\ref{app:heur}. As the dimension increases, the gap between the test error for the GLM-tron and hypentropy increases. This trend also holds for the $p$-norm algorithm for $n < 1200$, which begins to become brittle to the choice of hyperparameter for large $n$. As in the previous section, these results validate the predictions of Theorem~\ref{thm:stat_disc}, now with vector-valued outputs.

%% file: conc.tex
In this work, we studied the effect of optimization geometry on the statistical performance of generalized linear models trained with the square loss.
We obtained strong non-asymptotic guarantees that identify how the interplay between optimization geometry and feature space geometry can reduce dimensionality dependence of both the training and generalization errors.
We demonstrated the validity of our theoretical results on sparse vector and low-rank matrix recovery problems, where it was shown that pairing the optimization geometry with the feature space geometry as suggested by our analysis consistently led to improved out-of-sample performance. 

Single neurons and GLMs highlight important aspects of more complex deep models, so that our work provides insight into the observations by~\citet{azizan_2} that the choice of mirror descent potential affects the generalization performance of deep networks. Moreover, our results provide a quantitative characterization of this effect.

There are a number of natural directions for future work. A first goal is to classify the typical feature space geometry for neural networks on standard datasets. Given such a classification, a well-tailored potential function could be developed to improve generalization performance. A second question is whether there are pseudogradient methods suitable for multilayer architectures, and if they could lead to improved performance or a simpler analysis relative to gradient descent.

%% file: sim_details.tex
\label{sec:app:experiments}

\subsection{Projections}

\paragraph{Euclidean projection onto $\ell_1$-ball.}

For the GLM-tron algorithm, we use the following projection step
after every iteration:
\begin{align*}
    \argmin_{\bx : \norm{\bx}_1 \leq R} \norm{\bx - \by} \:. 
\end{align*}
The algorithm used to compute this is described in Figure 1 of \citet{duchi08projection}.

\paragraph{$\ell_p$ projection onto $\ell_p$-ball.}

For the $p$-norm algorithm, we use the following Bregman projection:
\begin{align*}
    \argmin_{\bx : \norm{\bx}_p \leq R} d_{\psi_p}(\bx, \by) \:.
\end{align*}
The solution is $\by$ for $\norm{\by}_p \leq R$ and
$\frac{\by}{\norm{\by}_p} R$ otherwise.
Note that we did not implement
the Bregman projection
\begin{align*}
    \argmin_{\bx : \norm{\bx}_1 \leq R} d_{\psi_p}(\bx, \by) \:.
\end{align*}
since we are not aware of an efficient  
(nearly linear time in dimension) algorithm for doing so.

\paragraph{Hypentropy Bregman projection onto $\ell_1$-ball.}

For the hypentropy algorithm, we use the following Bregman projection:
\begin{align*}
    \argmin_{\bx : \norm{\bx}_1 \leq R} d_{\psi_\beta}(\bx, \by) \:.
\end{align*}
To implement this projection,
we use the following bisection search algorithm communicated to us by Udaya Ghai, which was also used in~\citet{ghai20hypentropy}. Define the shrinkage function $s_\theta^\beta : \R^d \rightarrow \R^d$ as:
\begin{align*}
    s_\theta^\beta(\bx) = \mathrm{sign}(\bx) \max\left\{ \frac{\theta( \sqrt{\bx^2 + \beta^2} + \abs{\bx} )}{2} - \frac{ \sqrt{\bx^2 + \beta^2} - \abs{\bx} }{2\theta} , 0 \right\} \:,
\end{align*}
where the operations above are all elementwise.
One can show that there must exist a $\theta \in (0, 1]$ such that:
\begin{align*}
    s_\theta^\beta(\by) = \argmin_{\bx : \norm{\bx}_1 \leq R} d_{\psi_\beta}(\bx, \by) \:.
\end{align*}
From the above considerations, we can use bisection to search for a $\theta \in (0, 1]$ such that $$\norm{s_\theta^\beta(\by)}_1 = R.$$

\subsection{Hyperparameter values}

In this section, we list the hyperparameters that were gridded over for each figure.

\paragraph{Parameters for Figure~\ref{fig:sparse_vector_training_curve}}

\begin{center}
\begin{tabular}{ |c|c| } 
 \hline
Parameter & Values \\
\hline
 $\lambda$ & $\{1.0, 0.1, 0.01, 0.001\}$ \\
 \hline
 $\beta$ & $\{1.0, 10^{-1}, 10^{-2}, 10^{-3}, 10^{-4}\}$ \\
 \hline
\end{tabular}
\end{center}

\paragraph{Parameters for Figure~\ref{fig:sparse_vector_fixed_d}}

\begin{center}
\begin{tabular}{ |c|c| } 
 \hline
Parameter & Values \\
\hline
 $\lambda$ & $\{1.0, 0.5, 0.1, 0.05, 0.01, 0.005, 0.001\}$ \\
 \hline
 $\beta$ & $\{1.0, 10^{-1}, 10^{-2}, 10^{-3}, 10^{-4}\}$ \\
 \hline
 $p$ & $\{1.1, 1.2, 1.3, 1.4, 1.5\}$ \\
 \hline
\end{tabular}
\end{center}

\paragraph{Parameters for Figure~\ref{fig:sparse_vector_fixed_n}}
Same parameters as Figure~\ref{fig:sparse_vector_fixed_d}.

\paragraph{Parameters for Figure~\ref{fig:low_rank_fixed_d}}
\begin{center}
\begin{tabular}{ |c|c| } 
 \hline
Parameter & Values \\
\hline
 $\lambda$ & $\{0.1, 0.05, 0.01, 0.005, 0.001, 0.0005, 0.0001\}$ \\
 \hline
 $\beta$ & $\{1.0, 10^{-1}, 10^{-2}, 10^{-3}, 10^{-4}\}$ \\
 \hline
 $p$ & $\{1.1, 1.2, 1.3, 1.4, 1.5\}$ \\
 \hline
\end{tabular}
\end{center}

\paragraph{Parameters for Figure~\ref{fig:low_rank_fixed_ratio}}
Same parameters as Figure~\ref{fig:low_rank_fixed_d}.

\subsection{Heuristic argument for keeping $n=d$ in Figure~\ref{fig:low_rank_fixed_ratio}}
\label{app:heur}

Recall that the empirical loss is
\begin{align*}
    \wh{\varepsilon}(\wh{\bThet}) = \frac{1}{2nT} \sum_{t=0}^{T-1}\sum_{i=1}^{n} \norm{ \bx_{t+1}^i - \rho \bx_t^i - \sigma(\wh{\bThet} \bx_t^i) }^2 \:,
\end{align*}
while the pseudogradient $g(\wh{\bThet})$ is
\begin{align*}
    g(\wh{\bThet}) = \frac{1}{nT} \sum_{t=0}^{T-1}\sum_{i=1}^{n} (\sigma( \wh{\bThet} \bx_t^i  ) - \bx_{t+1}^i + \rho \bx_t^i  ) (\bx_t^i)^\T \:.
\end{align*}
A key term in the statistical bound for the Reflectron is the dual norm of the pseudogradient $g(\wh{\bThet})$. For the GLM-tron, this is the Frobenius norm $\norm{g(\wh{\bThet})}_F$, while for hypentropy this is the operator norm $\norm{g(\wh{\bThet})}$. The $p$-norm case is similiar to hypentropy for the purpose of this discussion, and we omit the details.

Estimating these norms in general is non-trivial due to both the nonlinearity of the
activation function and the time-dependence of the trajectory. Instead, we consider a simpler problem based on random matrices to heuristically understand relevant scalings with $n$ and $d$. In particular, we set $T=1$ and consider the $d \times d$ matrix $\bH$ defined as:
\begin{align*}
    \bH = \frac{1}{n} \sum_{i=1}^{n} \bx_i \bx_i^\T \:, \:\: \bx_i \sim N(\mathbf{0}, \bI) \:.
\end{align*}
Above, each of the $\bx_i$'s are independent. Let $\bX \in \R^{n\times d}$ be a matrix with $i$-th row given by $\bx_i$; then we have $\bH = \frac{1}{n} \bX^\T \bX$.
We first estimate a bound on $\E \norm{\bH}_F$ via Jensen's inequality
\begin{align*}
    \E \norm{\bH}_F &\leq \sqrt{ \E \norm{\bH}_F^2 } = \sqrt{\frac{1}{n^2} \Tr\left(\E \sum_{i,j} \bx_i \bx_i^\T \bx_j \bx_j^\T\right)} = \sqrt{\frac{1}{n^2} \sum_{i,j} \E \ip{\bx_i}{\bx_j}^2 } \\
    &= \sqrt{\frac{1}{n^2} \left(n \E\norm{\bx_1}^4 + n(n-1) \E \ip{\bx_1}{\bx_2}^2\right)} \\
    &= \sqrt{\frac{1}{n^2} \left(n(d^2 + 2d) + n(n-1) d\right)} \\
    &= \sqrt{ \frac{d^2}{n} + \left(1 + \frac{1}{n}\right)d} \\
    &\asymp \sqrt{d} + \frac{d}{\sqrt{n}} \:.
\end{align*}
On the other hand, $\norm{\bX} \lesssim \sqrt{n} + \sqrt{d}$ w.h.p.
Therefore,
\begin{align*}
  \norm{\bH} = \frac{1}{n} \norm{\bX}^2 \lesssim \frac{1}{n} (\sqrt{n} + \sqrt{d})^2 \asymp 1 + \frac{d}{n} \:.
\end{align*}
Now consider setting $n \asymp d$. Then as $n \to \infty$, we have that $\norm{\bH} \lesssim 1$ while $\norm{\bH}_F$ tends to $\infty$.

%% file: prior_rslts.tex
In this section, we present some results required for our proofs.

The following theorem gives a bound on the Rademacher complexity of a linear predictor, where the weights in the linear function class admit a bound in terms of a strongly convex potential function.
\begin{thm}[\citet{kakade2009_linear}]
\label{thm:kakade_linear}
Let $S$ be a closed convex set and let $\calX = \{\bx : \norm{\bx}_* \leq C\}$. Let $\psi : S \rightarrow \R$ be $\sigma$-strongly convex with respect to $\norm{\cdot}$ such that $\inf_{\bw \in S} \psi(\bw) = 0$. Define $\calW = \{\bw \in S : \psi(\bw) \leq W^2\}$, and let $\calF_{\calW} = \{\bx \mapsto \ip{\bw}{\bx} : \bw \in \calW\}$ Then,
\begin{equation*}
    \calR_n(\calF_\calW) \leq CW\sqrt{\frac{2}{\sigma n}}
\end{equation*}
\end{thm}

The following theorem is useful for bounding the Rademacher complexity of the generalized linear models considered in this work, as well as for bounding the generalization error in terms of the Rademacher complexity of a function class.
\begin{thm}[\citet{bartlett_gen}]
\label{thm:rad_glm}
Let $\phi : \mathbb{R} \rightarrow \mathbb{R}$ be $L_\phi$-Lipschitz, and assume that $\phi(0) = 0$. Let $\mathcal{F}$ be a class of functions. Then $\mathcal{R}_n(\phi\circ\mathcal{F}) \leq 2 L_\phi \mathcal{R}_n(\mathcal{F})$.
\end{thm}

The following theorem allows for a bound on the generalization error if bounds on the empirical risk and the Rademacher complexity of the function class are known.
\begin{thm}[\citet{bartlett_gen}]
\label{thm:ull_loss}
Let $\{\bx_i, y_i\}_{i=1}^n$ be an i.i.d. sample from a distribution $P$ over $\mathcal{X}\times\mathcal{Y}$ and let $\mathcal{L} : \mathcal{Y}' \times \mathcal{Y} \rightarrow \mathbb{R}$ be an $L$-Lipschitz and $b$-bounded loss function in its first argument. Let $\mathcal{F} = \{f \mid f : \mathcal{X} \rightarrow \mathcal{Y}'\}$ be a class of functions. For any positive integer $n \geq 0$ and any scalar $\delta \geq 0$,
\begin{equation*}
    \sup_{f \in \mathcal{F}}\left|\frac{1}{n}\sum_{i=1}^n\mathcal{L}(f(\bx_i), y_i) - \mathbb{E}_{(\bx, y)\sim P}\left[\mathcal{L}(f(\bx), y)\right]\right| \leq 4 L \mathcal{R}_n(\mathcal{F}) + 2 b \sqrt{\frac{2}{n}\log\left(\frac{1}{\delta}\right)}
\end{equation*}
with probability at least $1 - \delta$ over the draws of the $\{\bx_i, y_i\}$.
\end{thm}

The following lemma is a technical result from functional analysis which has seen widespread application in adaptive control theory~\citep{slot_li_book}.
\begin{lem}[Barbalat's Lemma]
\label{lem:barbalat}
Assume that $\bx : \R \rightarrow \R^n$ is such that $\bx \in \mathcal{L}_1$. If $\bx(t)$ is uniformly continuous in $t$, then $\lim_{t\rightarrow \infty}\bx(t) = 0$.
\end{lem}
Note that a sufficient condition for uniform continuity of $\bx(t)$ is that $\dot{\bx}(t) \in \mathcal{L}_\infty$.

The following two results will be used to obtain concentration inequalities in arbitrary $p$ norms for empirical averages of random vectors.
\begin{lem}
\label{lem:conc}
Let $\{X_i\}_{i=1}^n$ be random variables in a Banach space $\calX$ equipped with a norm $\norm{\cdot}$ such that $\norm{X_i} \leq C$. Then for any $\delta > 0$, with probability at least $1-\delta$,
\begin{equation*}
    \abs{\norm{\frac{1}{n}\sum_{i=1}^n X_i} - \mathbb{E}\left[\norm{\frac{1}{n}\sum_{i=1}^n X_i}\right]} \leq \sqrt{\frac{2C^2}{n}\log(2/\delta)}
\end{equation*}
\end{lem}
\begin{proof}
Observe that by the reverse triangle inequality, $f(X_1, X_2, \hdots, X_n) = \norm{\sum_{i=1}^n X_i}$ satisfies the bounded differences inequality with uniform bound $2C$.
\end{proof}

\begin{lem}
\label{lem:exp}
Let $\{\bX_i\}_{i=1}^n$ be random vectors in Euclidean space $\bX_i \in \calX \subseteq \R^d$ such that $\norm{\bX_i}_p \leq C$ and $\Exp{\bX_i} = 0$ with $p \in [1, \infty]$. Then the following bound holds
\begin{equation*}
    \Exp{\norm{\frac{1}{n}\sum_{i=1}^n \bX_i}_p} \leq
    \begin{cases}
    \frac{d^{2/p - 1} 2^{1/2}C}{\sqrt{n}} & p \in [1, 2)\\
    \frac{C}{\sqrt{n(q-1)}} & p \in [2, \infty)\\
    4C\sqrt{\frac{\log(d)}{n}} & p = \infty
    \end{cases}
\end{equation*}
\end{lem}
\begin{proof}
Let $\epsilon_i$ denote a Rademacher random variable. By a standard symmetrization argument,
\begin{align*}
    \Expsub{\bX_i}{\norm{\sum_{i=1}^n\left(\bX_i - \Expsub{\bX_i}{\bX_i}\right)}_p} &\leq 2\Expsub{\bX_i, \epsilon_i}{\norm{\sum_{i=1}^n \epsilon_i \bX_i}_p}.
\end{align*}
Let $\mathcal{F} = \{\bx \mapsto \ip{\bx}{\bw} : \norm{\bw}_q \leq 1\}$ with $\frac{1}{q} + \frac{1}{p} = 1$. Observe that by definition of the dual norm
\begin{align*}
    \Expsub{\bX_i, \epsilon_i}{\norm{\sum_{i=1}^n \epsilon_i \bX_i}_p} &= n\mathcal{R}_n(\mathcal{F}).
\end{align*}
By Theorem~\ref{thm:kakade_linear}, noting that $\norm{\cdot}_q^2$ is $\frac{1}{2(q-1)}$-strongly convex with respect to $\norm{\cdot}_q$ for $q\in (1, 2]$, we then have that
\begin{equation*}
    \Exp{\norm{\frac{1}{n}\sum_{i=1}^n \bX_i}_p} \leq \frac{C}{\sqrt{n(q-1)}},
\end{equation*}
where $q\in(1, 2]$ implies that $p\in[2, \infty)$. 

Now consider the case $p = \infty$. Because each $\norm{\bX_i}_\infty \leq C$, each component of each $\bX_i$ is sub-Gaussian. Hence,
\begin{equation*}
    \Exp{\norm{\frac{1}{n}\sum_{i=1}^n \bX_i}_{\infty}} \leq 4C\sqrt{\frac{\log(d)}{n}}.
\end{equation*}
Last, consider $p \in [1, 2)$. Then we have the elementary bound via equivalence of norms
\begin{align*}
    \Exp{\norm{\frac{1}{n}\sum_{i=1}^n \bX_i}_p} &\leq d^{1/p - 1/2}\Exp{\norm{\frac{1}{n}\sum_{i=1}^n \bX_i}_2} \leq \frac{d^{2/p - 1} 2^{1/2}C}{\sqrt{n}}.
\end{align*}
This completes the proof.
\end{proof}
To analyze our discrete-time iterations, we require three basic properties of the Bregman divergence.
\begin{lem}[Bregman three-point identity]
Let $\psi:\calM\rightarrow\mathbb{R}^p$ denote a $\sigma$-strongly convex function with respect to some norm $\norm{\cdot}$. Then for all $\bx, \by, \bz \in \calM$,
\begin{equation}
\label{eqn:breg_3pt}
\ip{\nabla\psi(\bx) - \nabla\psi(\by)}{\bx - \bz} = \bregd{\bx}{\by} + \bregd{\bz}{\bx} - \bregd{\bz}{\by}.
\end{equation}
\end{lem}
\begin{lem}[Generalized Pythagorean Theorem]
Let $\psi : \calM\rightarrow\mathbb{R}$ denote a $\sigma$-strongly convex function with respect to some norm $\norm{\cdot}$. Let $\bx_0 \in \calM$ and let $\bx^* = \Pi_{\calC}^\psi(\bx_0)$ be its projection onto a closed and convex set $\calC$. Then for any $\by \in \calC$,
\begin{equation}
    \label{eqn:pythag_gen}
    \bregd{\by}{\bx_0} \geq \bregd{\by}{\bx^*} + \bregd{\bx^*}{\bx_0}.
\end{equation}
\end{lem}
\begin{lem}[Bregman duality]
Let $\psi : \calM \rightarrow \R$ denote a $\sigma$-strongly convex function with respect to some norm $\norm{\cdot}$, and let $\psi^*$ denote its Fenchel conjugate. Then $\psi^*$ is $1/\sigma$-smooth with respect to $\norm{\cdot}_*$, and moreover
\begin{equation}
    \label{eqn:breg_dual}
    \bregd{\bx}{\by} = \bregd{\nabla\psi^*(\by)}{\nabla\psi^*(\bx)}
\end{equation}
\end{lem}
To obtain fast rates in the realizable online learning setting, we require the following martingale Bernstein bound, which has been used in similar analyses prior to this work~\citep{telgarsky, agnostic_neuron}.
\begin{lem}[\citet{beygel}]
\label{lem:martingale}
Let $\{Y_t\}_{t=1}^\infty$ be a martingale adapted to the filtration $\{\calF_t\}_{t=1}^\infty$. Let $\{D_t\}_{t=1}^{\infty}$ be the corresponding martingale difference sequence. Define
\begin{equation*}
    V_t = \sum_{k=1}^t\Exp{D_k^2|\calF_{k-1}},
\end{equation*}
and assume that $D_t \leq R$ almost surely. Then for any $\delta \in (0, 1)$, with probability at least $1-\delta$,
\begin{equation*}
    Y_t \leq R\log(1/\delta) + (e-2)V_t/R.
\end{equation*}
\end{lem}

%% file: proofs.tex
\subsection{Proof of Theorem~\ref{thm:stat_disc}}
To make progress in the general setting, we require a definition of a modified error function for a parametric hypothesis $h(\bx) = u\left(\ip{\wh{\bthet}}{\bx}\right)$. The empirical version over the dataset $\wh{H}(h)$ is defined analogously.
\begin{equation*}
    H(h) = \Exp{\left(u\left(\ip{\wh{\bthet}}{\bx}\right) - u\left(\ip{\bthet}{\bx}\right)\right)^2\xi\left(\wh{\bthet}, \bx\right)}
\end{equation*}
Intuitively, under Assumption~\ref{assmp:frei}, we can relate $H$ to $\varepsilon$. The following lemma makes this rigorous, and is adapted from~\citet{agnostic_neuron}. The proof is a trivial modification of the proof given in their work.
\begin{lem}
\label{lem:frei}
Let $\xi$ satisfy Assumption~\ref{assmp:frei}. Let $h$ denote a parametric hypothesis of the form $h(\bx) = u\left(\ip{\wh{\bthet}}{\bx}\right)$. Then if $\norm{\wh{\bthet}} \leq a$ and $\norm{\bx}_* \leq b$, we have the bound $\wh{\varepsilon}(h) \leq \wh{H}(h)/\gamma$ where $\gamma$ is a fixed constant defined in Assumption~\ref{assmp:frei}.
\end{lem}
We now begin the proof of Theorem~\ref{thm:stat_disc}.
\begin{proof}
By the Bregman three-point identity \eqref{eqn:breg_3pt}, with $\bz = \bthet$, $\bx = \wh{\bphi}_{t+1}$, and $\by = \wh{\bthet}_t$,
\begin{align*}
    \bregd{\bthet}{\wh{\bphi}_{t+1}} &= \bregd{\bthet}{\wh{\bthet}_t} - \bregd{\wh{\bphi}_{t+1}}{\wh{\bthet}_t} + \ip{\nabla\psi(\wh{\bphi}_{t+1}) - \nabla\psi(\wh{\bthet}_t)}{\wh{\bphi}_{t+1} - \bthet},\\
    &= \bregd{\bthet}{\wh{\bthet}_t} - \bregd{\wh{\bphi}_{t+1}}{\wh{\bthet}_t} - \ip{\frac{\lambda}{n}\sum_{i=1}^n\left(u\left(\left\langle\bx_i, \wh{\bthet}_t\right\rangle\right) - y_i\right)\xi\left(\wh{\bthet}_t, \bx_i\right)\bx_i}{\wh{\bphi}_{t+1} - \bthet},\\
    &= \bregd{\bthet}{\wh{\bthet}_t} - \bregd{\wh{\bphi}_{t+1}}{\wh{\bthet}_t} - \ip{\frac{\lambda}{n}\sum_{i=1}^n\left(u\left(\left\langle\bx_i, \wh{\bthet}_t\right\rangle\right) - y_i\right)\xi\left(\wh{\bthet}_t, \bx_i\right)\bx_i}{\wh{\bthet}_{t} - \bthet}\\
    &\phantom{=} - \ip{\frac{\lambda}{n}\sum_{i=1}^n\left(u\left(\left\langle\bx_i, \wh{\bthet}_t\right\rangle\right) - y_i\right)\xi\left(\wh{\bthet}_t, \bx_i\right)\bx_i}{\wh{\bphi}_{t+1} - \wh{\bthet}_{t}}.
\end{align*}
After grouping the second and final terms in the above expression as in the proof of Theorem~\ref{thm:stat_disc}, the iteration becomes
\begin{align*}
    \bregd{\bthet}{\wh{\bphi}_{t+1}} &= \bregd{\bthet}{\wh{\bthet}_t} + \bregd{\wh{\bthet}_{t}}{\wh{\bphi}_{t+1}} + \frac{\lambda}{n}\ip{\sum_{i=1}^n\left(u\left(\ip{\bx_i} {\wh{\bthet}_t}\right) - y_i\right)\xi\left(\wh{\bthet}_t, \bx_i\right)\bx_i}{\bthet - \wh{\bthet}_{t}}.
\end{align*}
By the generalized Pythagorean Theorem \eqref{eqn:pythag_gen},
\begin{align}
    \bregd{\bthet}{\wh{\bthet}_{t+1}} &\leq \bregd{\bthet}{\wh{\bthet}_t} + \bregd{\wh{\bthet}_{t}}{\wh{\bphi}_{t+1}} \nonumber\\
    &\phantom{=}\hspace{.35in} + \frac{\lambda}{n}\ip{\sum_{i=1}^n\left(u\left(\ip{\bx_i} {\wh{\bthet}_t}\right) - y_i\right)\xi\left(\wh{\bthet}_t, \bx_i\right)\bx_i}{\bthet - \wh{\bthet}_{t}}.
    \label{eqn:breg_inc_simp}
\end{align}
By duality \eqref{eqn:breg_dual}, we may replace $\bregd{\wh{\bthet}_t}{\wh{\bphi}_{t+1}}$ by $\bregd{\nabla\psi^*\left(\wh{\bphi}_{t+1}\right)}{\nabla\psi^*\left(\wh{\bthet}_t\right)}$,
\begin{align*}
    \bregd{\bthet}{\wh{\bthet}_{t+1}} &\leq \bregd{\bthet}{\wh{\bthet}_t} + \bregdd{\wh{\bphi}_{t+1}}{\wh{\bthet}_{t}} \\
    &\phantom{=} \hspace{.35in}+ \frac{\lambda}{n}\ip{\sum_{i=1}^n\left(u\left(\ip{\bx_i} {\wh{\bthet}_t}\right) - y_i\right)\xi\left(\wh{\bthet}_t, \bx_i\right)\bx_i}{\bthet - \wh{\bthet}_{t}}.
\end{align*}
Because $\psi$ is $\sigma$-strongly convex with respect to $\Vert\cdot\Vert$, $\psi^*$ is $\frac{1}{\sigma}$-smooth with respect to $\Vert\cdot\Vert_*$. Thus,
\begin{align}
        \bregd{\bthet}{\wh{\bthet}_{t+1}} &\leq \bregd{\bthet}{\wh{\bthet}_t} + \frac{1}{2\sigma}\left\Vert\nabla\psi\left(\wh{\bphi}_{t+1}\right) - \nabla\psi\left(\wh{\bthet}_t\right)\right\Vert_*^2 \nonumber\\
        &\phantom{=} \hspace{.35in}+ \frac{\lambda}{n}\ip{\sum_{i=1}^n\left(u\left(\ip{\bx_i} {\wh{\bthet}_t}\right) - y_i\right)\xi\left(\wh{\bthet}_t, \bx_i\right)\bx_i}{\bthet - \wh{\bthet}_{t}}\nonumber,\\
        &= \bregd{\bthet}{\wh{\bthet}_t} + \frac{\lambda^2}{2\sigma}\left\Vert\frac{1}{n}\sum_{i=1}^n\left(u\left(\left\langle\bx_i, \wh{\bthet}_t\right\rangle\right) - y_i\right)\xi\left(\wh{\bthet}_t, \bx_i\right)\bx_i\right\Vert_*^2\nonumber\\
        &\phantom{=} + \frac{\lambda}{n}\ip{\sum_{i=1}^n\left(u\left(\ip{\bx_i} {\wh{\bthet}_t}\right) - y_i\right)\xi\left(\wh{\bthet}_t, \bx_i\right)\bx_i}{\bthet - \wh{\bthet}_{t}}.
        \label{eqn:interim_step}
\end{align}
Above, we applied $\frac{1}{\sigma}$-smoothness and then used \eqref{eqn:refl_disc} to express the increment in $\nabla\psi$. The second term in \eqref{eqn:interim_step} can be bounded as
\begin{align*}
    &\left\Vert\frac{1}{n}\sum_{i=1}^n\left(u\left(\left\langle\bx_i, \wh{\bthet}_t\right\rangle\right) - y_i\right)\xi\left(\wh{\bthet}_t, \bx_i\right)\bx_i\right\Vert_*^2\\
    &\leq 2\left\Vert\frac{1}{n}\sum_{i=1}^n\left(u\left(\left\langle\bx_i, \wh{\bthet}_t\right\rangle\right) - u\left(\left\langle\bx_i, \bthet\right\rangle\right)\right)\xi\left(\wh{\bthet}_t, \bx_i\right)\bx_i\right\Vert_*^2 + 2\left\Vert\frac{1}{n}\sum_{i=1}^n\left(u\left(\left\langle\bx_i, \bthet\right\rangle\right) - y_i\right)\xi\left(\wh{\bthet}_t, \bx_i\right)\bx_i\right\Vert_*^2.
\end{align*}
By Jensen's inequality, and using that $\xi\left(\wh{\bthet}_t, \bx_i\right) \leq B$,
\begin{align*}
    \left\Vert\frac{1}{n}\sum_{i=1}^n\left(u\left(\left\langle\bx_i, \wh{\bthet}_t\right\rangle\right) - u\left(\left\langle\bx_i, \bthet\right\rangle\right)\right)\xi\left(\wh{\bthet}_t, \bx_i\right)\bx_i\right\Vert_*^2 &\leq \frac{1}{n}\sum_{i=1}^n\norm{\left(u\left(\ip{\wh{\bthet}}{\bx_i}\right) - u\left(\ip{\bthet}{\bx_i}\right)\right)\xi\left(\wh{\bthet}_t, \bx_i\right)\bx_i}_*^2,\\
    &= \frac{1}{n}\sum_{i=1}^n\left(u\left(\ip{\wh{\bthet}}{\bx_i}\right) - u\left(\ip{\bthet}{\bx_i}\right)\right)^2\xi\left(\wh{\bthet}_t, \bx_i\right)^2\norm{\bx_i}_*^2,\\
    &\leq C^2B\wh{H}(h_t).
\end{align*} 
By assumption, $\Vert\frac{1}{n}\sum_{i=1}^n\left(y_i - u(\left\langle\bx_i, \bthet\right\rangle)\right)\xi\left(\wh{\bthet}_t, \bx_i\right)\bx_i\Vert_* \leq \eta$. Combining this with the above, we find that
\begin{align*}
    \left\Vert\frac{1}{n}\sum_{i=1}^n\left(u\left(\left\langle\bx_i, \wh{\bthet}_t\right\rangle\right) - y_i\right)\xi\left(\wh{\bthet}_t, \bx_i\right)\bx_i\right\Vert_*^2
    &\leq 2\left(C^2B\wh{H}(h_t) + \eta^2\right).
\end{align*}
By an induction argument identical to that used in the proof of Theorem~\ref{thm:stat_disc}, the iteration of the Bregman divergence between the Bayes-optimal parameters and the parameters of our hypothesis becomes
\begin{align*}
    \bregd{\bthet}{\wh{\bthet}_{t+1}} &\leq \bregd{\bthet}{\wh{\bthet}_t} - \lambda\wh{H}(h_t)\left(\frac{1}{L} - \frac{\lambda B C^2}{\sigma}\right) + \eta\lambda\left(\frac{\lambda\eta}{\sigma} + \sqrt{\frac{2\psi(\bthet)}{\sigma}}\right).
\end{align*}
Assume that $\eta \leq \sqrt{\frac{2\psi(\bthet)}{\sigma}}$. For $\lambda \leq \frac{\sigma}{2BC^2L}$, we have
\begin{equation*}
    \bregd{\bthet}{\wh{\bthet}_{t+1}} \leq \bregd{\bthet}{\wh{\bthet}_t} - \frac{\lambda}{2L}\wh{H}(h_t) + \eta\lambda\sqrt{\frac{2\psi(\bthet)}{\sigma}}\left(\frac{2BC^2L + 1}{2BC^2L}\right).
 \end{equation*}
Thus, at each iteration we either have the decrease condition
\begin{equation*}
    \bregd{\bthet}{\wh{\bthet}_{t+1}} - \bregd{\bthet}{\wh{\bthet}_t} \leq -\eta\lambda\sqrt{\frac{2\psi(\bthet)}{\sigma}}\left(\frac{2BC^2L + 1}{2BC^2L}\right)
\end{equation*}
or the error bound
\begin{equation*}
    \wh{H}(h_t) < 4L\eta\sqrt{\frac{2\psi(\bthet)}{\sigma}}\left(\frac{2BC^2L + 1}{2BC^2L}\right).
\end{equation*}
In the former case there can be at most
\begin{equation*}
    t_f = \frac{\bregd{\bthet}{\wh{\bthet}(0)}}{\eta\lambda\sqrt{\frac{2\psi(\bthet)}{\sigma}}\left(\frac{2BC^2L + 1}{2BC^2L}\right)} \leq \frac{1}{\lambda}\sqrt{\frac{\sigma\psi(\bthet)}{2\eta^2}}
\end{equation*}
iterations before $\wh{H}(h_t) \leq 4L\eta\sqrt{\frac{2\psi(\bthet)}{\sigma}}\left(\frac{2BC^2L + 1}{2BC^2L}\right)$. Furthermore, note that $\norm{\wh{\bthet}(t)} \leq \sqrt{\frac{2\psi(\bthet)}{\sigma}} + \norm{\bthet} \leq (1+W)\sqrt{\frac{2\psi(\bthet)}{\sigma}}$. Then by Lemma~\ref{lem:frei}, $\wh{\varepsilon}(h_t) \leq \frac{4L\eta}{\gamma}\sqrt{\frac{2\psi(\bthet)}{\sigma}}\left(\frac{2BC^2L + 1}{2BC^2L}\right)$ where $\gamma$ corresponds to $a =  \left(1+W\right)\sqrt{\frac{2\psi(\bthet)}{\sigma}}$ and $b = C$ in Lemma~\ref{lem:frei}. The conclusion of the theorem now follows by application of Theorem~\ref{thm:ull_loss} to transfer the bound on $\wh{\varepsilon}(h_t)$ to $\varepsilon(h_t)$.
\end{proof}
\subsection{Proof of Corollary~\ref{cor:pq_1}}
\begin{proof}
Note that $\calF \subseteq \left\{\bx\mapsto\ip{\bw}{\bx} : \norm{\bw}_q \leq W\left(1 + \frac{1}{\sqrt{q-1}}\right)\right\}$. Hence by Theorem~\ref{thm:kakade_linear}, $\calR_n(\calF) \leq \frac{CW}{\sqrt{n(q-1)}}\left(1 + \frac{1}{\sqrt{q-1}}\right)$. By Lemmas~\ref{lem:conc} and~\ref{lem:exp}, $\eta = C\left(\sqrt{\frac{2\log(4/\delta)}{n}} + \frac{1}{\sqrt{n(q-1)}}\right)$.
\end{proof}
\subsection{Proof of Corollary~\ref{cor:pq_2}}
\begin{proof}
Observe that we have the inclusion $$\calF \subseteq \left\{\bx\mapsto\ip{\bw}{\bx} : \norm{\bw}_1 \leq W\left(1 + \sqrt{3\log(d)}\right)\right\} \subseteq \left\{\bx\mapsto\ip{\bw}{\bx} : \norm{\bw}_q \leq W\left(1 + \sqrt{3\log(d)}\right)\right\}.$$ Hence $\calR_n(\calF) \leq \frac{CW(1+\sqrt{3\log d})^2}{n^{1/2}}$ by Theorem~\ref{thm:kakade_linear}. By Lemmas~\ref{lem:conc} and~\ref{lem:exp}, $$\eta = C\left(\sqrt{\frac{2\log(4/\delta)}{n}} + 4\sqrt{\frac{\log(d)}{n}}\right)$$.
\end{proof}
\subsection{Proof of Corollary~\ref{cor:ent}}
\begin{proof}
Note that $\calF \subseteq \left\{\bx \mapsto \ip{\bw}{\bx} : \psi(\bw) \leq \log(d)\right\}$ and $\calR_n \leq C\sqrt{\frac{2\log d}{n}}$. By Lemmas~\ref{lem:conc} and~\ref{lem:exp}, $\eta = C\left(\sqrt{\frac{2\log(4/\delta)}{n}} + 4\sqrt{\frac{\log(d)}{n}}\right)$.
\end{proof}
\subsection{Proof of Lemma~\ref{lem:conv_disc}}
\begin{proof}
From \eqref{eqn:interim_step}, we have a bound on the iteration for the Bregman divergence between the interpolating parameters and the current parameter estimates,
\begin{align*}
        \bregd{\bthet}{\wh{\bthet}_{t+1}} &\leq \bregd{\bthet}{\wh{\bthet}_t} + \frac{\lambda^2}{2\sigma}\left\Vert\frac{1}{n}\sum_{i=1}^n\left(u\left(\left\langle\bx_i, \wh{\bthet}_t\right\rangle\right) - y_i\right)\bx_i\xi\left(\wh{\bthet}, \bx_i\right)\right\Vert_2^2\nonumber\\
        &\phantom{=} + \frac{\lambda}{n}\ip{\sum_{i=1}^n\left(u\left(\ip{\bx_i} {\wh{\bthet}_t}\right) - y_i\right)\xi\left(\wh{\bthet}, \bx_i\right)\bx_i}{\bthet - \wh{\bthet}_{t}}.
\end{align*}
Under the realizability assumption of the lemma, we may bound the second term above as
\begin{equation*}
    \frac{\lambda^2}{2\sigma}\left\Vert\frac{1}{n}\sum_{i=1}^n\left(u\left(\left\langle\bx_i, \wh{\bthet}_t\right\rangle\right) - y_i\right)\xi\left(\wh{\bthet}, \bx_i\right)\bx_i\right\Vert_2^2 \leq \frac{\lambda^2 C^2B}{2\sigma}\wh{H}(h_t).
\end{equation*}
We may similarly bound the final term, exploiting monotonicity and Lipschitz continuity of $u(\cdot)$, as
\begin{equation*}
    \frac{\lambda}{n}\ip{\sum_{i=1}^n\left(u\left(\ip{\bx_i} {\wh{\bthet}_t}\right) - y_i\right)\xi\left(\wh{\bthet}, \bx_i\right)\bx_i}{\bthet - \wh{\bthet}_{t}} \leq -\frac{\lambda}{L}\wh{H}(h_t).
\end{equation*}
Putting these together, we have the refined bound on the iteration
\begin{equation*}
    \bregd{\bthet}{\wh{\bthet}_{t+1}} \leq \bregd{\bthet}{\wh{\bthet}_t} + \lambda\left(\frac{\lambda C^2B}{2\sigma} - \frac{1}{L}\right)\wh{H}(h_t).
\end{equation*}
Let $0 < \alpha < 1$. For $\lambda \leq \frac{2\sigma\left(1-\alpha\right)}{C^2BL}$,
\begin{equation*}
    \bregd{\bthet}{\wh{\bthet}_{t+1}} \leq \bregd{\bthet}{\wh{\bthet}_t} - \frac{\lambda \alpha}{L}\wh{H}(h_t).
\end{equation*}
Note that this shows $\bregd{\bthet}{\wh{\bthet}_t}\leq\bregd{\bthet}{\wh{\bthet}_1}$ for all $t$, so that $\norm{\wh{\bthet}_t} \leq \norm{\bthet} + \sqrt{\frac{2\psi(\bthet)}{\sigma}}$. Summing both sides of the above inequality from $t=1$ to $T$ reveals that
\begin{equation*}
    \sum_{t=1}^{T} \wh{H}(h_t) \leq \frac{L}{\lambda\alpha}\left(\bregd{\bthet}{\bthet_1} - \bregd{\bthet}{\wh{\bthet}_{T+1}}\right) \leq \frac{L}{\lambda\alpha}\bregd{\bthet}{\bthet_1}.
\end{equation*}
Because $T$ was arbitrary and the upper bound is independent of $T$, $\sum_{t=1}^\infty \wh{H}(h_t)$ exists and hence $\wh{H}(h_t) \rightarrow 0$ as $t\rightarrow \infty$. Furthermore,
\begin{equation*}
    \min_{t'\in [1, T]}\left\{\wh{H}(h_{t'})\right\}T = \sum_{t=1}^{T}\min_{t'\in [1, T]}\left\{\wh{H}(h_{t'})\right\} \leq \sum_{t=1}^{T}\wh{H}(h_{t}) \leq \frac{L}{\lambda\alpha}\bregd{\bthet}{\bthet_1},
\end{equation*}
so that $\min_{t'\in [1, T]}\left\{\wh{H}(h_{t'})\right\} \leq \frac{L\bregd{\bthet}{\bthet_1}}{\alpha\lambda T}$. By taking $\alpha\to 0$, we obtain the requirement $\lambda < \frac{2\sigma}{C^2 BL}$. To conclude the proof, apply Lemma~\ref{lem:frei} with $a = \sqrt{\frac{2\psi(\bthet)}{\sigma}} + \norm{\bthet}$ and $b = C$.
\end{proof}
\subsection{Proof of Theorem~\ref{thm:imp_reg_disc}}
The proof discretizes the proof of Theorem~\ref{thm:imp_reg}, and is similar to the proof of implicit regularization for mirror descent due to~\citet{azizan_1}.
\begin{proof}
Let $\bar{\bthet} \in \mathcal{A}$ be arbitrary. From \eqref{eqn:breg_inc_simp},
\begin{align*}
    \bregd{\bar{\bthet}}{\wh{\bthet}_{t+1}} &\leq \bregd{\bar{\bthet}}{\wh{\bthet}_t} + \bregd{\wh{\bthet}_{t}}{\wh{\bphi}_{t+1}} + \frac{\lambda}{n}\sum_{i=1}^n\left(u\left(\ip{\bx_i} {\wh{\bthet}_t}\right) - y_i\right)\xi\left(\wh{\bthet}_t, \bx_i\right)\ip{\bx_i}{\bar{\bthet} - \wh{\bthet}_{t}},\\
    &= \bregd{\bar{\bthet}}{\wh{\bthet}_t} + \bregd{\wh{\bthet}_{t}}{\wh{\bphi}_{t+1}}\\
    &\phantom{=} \hspace{.5in} + \frac{\lambda}{n}\sum_{i=1}^n\left(u\left(\ip{\bx_i} {\wh{\bthet}_t}\right) - y_i\right)\xi\left(\wh{\bthet}_t, \bx_i\right)\left(u^{-1}\left(y_i\right) - \ip{\bx_i}{\wh{\bthet}_{t}}\right),
\end{align*}
where we have used that $\overline{\bthet} \in \mathcal{A}$ and applied invertibility of $u(\cdot)$. Summing both sides from $t=1$ to $\infty$,
\begin{align*}
    \bregd{\overline{\bthet}}{\wh{\bthet}_{\infty}} &\leq \bregd{\overline{\bthet}}{\wh{\bthet}_1} + \sum_{t=1}^{\infty}\bregd{\wh{\bthet}_{t}}{\wh{\bphi}_{t+1}} \\
    &\phantom{=} \hspace{.5in} + \frac{\lambda}{n}\sum_{i=1}^n\sum_{t=1}^{\infty}\left(u\left(\ip{\bx_i} {\wh{\bthet}_t}\right) - y_i\right)\xi\left(\wh{\bthet}_t, \xi_i\right)\left(u^{-1}\left(y_i\right) - \ip{\bx_i}{\wh{\bthet}_t}\right).
\end{align*}
The above relation is true for any $\overline{\bthet} \in \mathcal{A}$. Furthermore, the only dependence of the right-hand side on $\overline{\bthet}$ is through the first Bregman divergence. Hence the $\argmin$ of the two Bregman divergences involving $\overline{\bthet}$ must be equal, which shows that $\wh{\bthet}_{\infty} = \argmin_{\overline{\bthet}\in\mathcal{A}}\bregd{\overline{\bthet}}{\wh{\bthet}_1}$.
Choosing $\wh{\bthet}_1 = \argmin_{\bw \in \calC\cap\calM}\psi(\bw)$ completes the proof.
\end{proof}
\subsection{Proof of Theorem~\ref{thm:online_noise}}
\begin{proof}
Let $\xi_t = \xi\left(\wh{\bthet}_t, \bx_t\right)$. From \eqref{eqn:interim_step} adapted to the stochastic optimization setting, we have the bound
\begin{align*}
    \bregd{\bthet}{\wh{\bthet}_{t+1}} &\leq \bregd{\bthet}{\wh{\bthet}_t} + \frac{\lambda^2}{2\sigma}\left\Vert\left(u\left(\ip{\bx_t}{\wh{\bthet}_t}\right) - y_t\right)\bx_t\xi_t\right\Vert_*^2 + \lambda\ip{\left(u\left(\ip{\bx_t}{\wh{\bthet}_t}\right) - y_t\right)\bx_t\xi_t}{\bthet - \wh{\bthet}_t}.
\end{align*}
Note that we can write
\begin{align*}
    \left(u\left(\ip{\bx_t}{\wh{\bthet}_t}\right) - y_t\right)^2 &= \left(u\left(\ip{\bx_t}{\wh{\bthet}_t}\right) - u\left(\ip{\bx_t}{\bthet}\right)\right)^2 + \left(u\left(\ip{\bx_t}{\bthet}\right) - y_t\right)^2\\
    &\phantom{=} \hspace{.5in} + 2\left(u\left(\ip{\bx_t}{\wh{\bthet}_t}\right) - u\left(\ip{\bx_t}{\bthet}\right)\right)\left(u\left(\ip{\bx_t}{\bthet}\right) - y_t\right).
\end{align*}
Using that $u$ is nondecreasing and $L$-Lipschitz,
\begin{align*}
    \ip{\left(u\left(\ip{\bx_t}{\wh{\bthet}_t}\right) - y_t\right)\bx_t\xi_t}{\bthet - \wh{\bthet}_t} &\leq -\frac{1}{L}\left(u\left(\ip{\bx_t}{\wh{\bthet}_t}\right) - u\left(\ip{\bx_t}{\bthet}\right)\right)^2\xi_t\\
    &\phantom{=} \hspace{.4in} + \left(u\left(\ip{\bx_t}{\bthet}\right) - y_t\right)\xi_t\ip{\bx_t}{\bthet-\wh{\bthet}_t}.
\end{align*}
Putting these together, we conclude the bound,
\begin{align*}
    \bregd{\bthet}{\wh{\bthet}_{t+1}} &\leq \bregd{\bthet}{\wh{\bthet}_t} - \lambda\left(\frac{1}{L} - \frac{\lambda C^2B}{2\sigma}\right)\left(u\left(\ip{\bx_t}{\wh{\bthet}_t}\right) - u\left(\ip{\bx_t}{\bthet}\right)\right)^2\xi_t \\
    &\phantom{=} + \lambda \xi_t\left(u\left(\ip{\bx_t}{\bthet}\right) - y_t\right)\left(\ip{\bx_t}{\bthet - \wh{\bthet}_t} + \frac{\lambda C^2B}{\sigma}\left(u\left(\ip{\bx_t}{\wh{\bthet}_t}\right) - u\left(\ip{\bx_t}{\bthet}\right)\right)\right)\\
    &\phantom{=} + \frac{\lambda^2 C^2B^2}{2\sigma} \left(u\left(\ip{\bx_t}{\bthet}\right) - y_t\right)^2.
\end{align*}
Summing both sides from $t=1$ to $T$,
\begin{align*}
    \bregd{\bthet}{\wh{\bthet}_{T+1}} &\leq \bregd{\bthet}{\wh{\bthet}_1} - \lambda\left(\frac{1}{L} - \frac{\lambda C^2B}{2\sigma}\right)\sum_{t=1}^{T}\left(u\left(\ip{\bx_t}{\wh{\bthet}_t}\right) - u\left(\ip{\bx_t}{\bthet}\right)\right)^2\xi_t \\
    &\phantom{=} + \lambda \sum_{t=1}^{T}\xi_t\left(u\left(\ip{\bx_t}{\bthet}\right) - y_t\right)\left(\ip{\bx_t}{\bthet - \wh{\bthet}_t} + \frac{\lambda C^2B}{\sigma}\left(u\left(\ip{\bx_t}{\wh{\bthet}_t}\right) - u\left(\ip{\bx_t}{\bthet}\right)\right)\right)\\
    &\phantom{=} + \frac{\lambda^2 C^2B^2}{2\sigma} \sum_{t=1}^{T}\left(u\left(\ip{\bx_t}{\bthet}\right) - y_t\right)^2.
\end{align*}
Define the filtration $\{\calF_t = \sigma(\bx_1, y_1, \bx_2, y_2, \hdots, \bx_t, y_t, \bx_{t+1})\}_{t=1}^\infty$, and note that
\begin{align*}
    D^{(1)}_t &= \xi_t\left(u\left(\ip{\bx_t}{\bthet}\right) - y_t\right)\ip{\bx_t}{\bthet - \wh{\bthet}_t},\\
    D^{(2)}_t &= \xi_t\left(u\left(\ip{\bx_t}{\bthet}\right) - y_t\right)\left(u\left(\ip{\bx_t}{\wh{\bthet}_t}\right) - u\left(\ip{\bx_t}{\bthet}\right)\right),
\end{align*}
are martingale difference sequences adapted to $\{\calF_t\}$. Furthermore, note that $|D_t^{(1)}| \leq CBR$ and $|D_t^{(2)}| \leq LCBR$ almost surely where $R = \diam(C)$. Hence, by an Azuma-Hoeffding bound, with probability at least $1-\delta/3$,
\begin{align*}
    \sum_{t=1}^{T}D_t^{(1)} &\leq \sqrt{CBRT\log(6/\delta)},\\
    \sum_{t=1}^{T}D_t^{(2)} &\leq \sqrt{LCBRT\log(6/\delta)}.
\end{align*}
The variance term is trivially bounded almost surely,
\begin{equation*}
    \sum_{t=1}^{T}\left(u\left(\ip{\bx_t}{\bthet}\right) - y_t\right)^2 \leq T.
\end{equation*}
Putting these bounds together and rearranging, we conclude that with probability at least $1 - 2\delta/3$,
\begin{align*}
    \lambda\left(\frac{1}{L} - \frac{\lambda C^2B}{2\sigma}\right)&\sum_{t=1}^{T}\left(u\left(\ip{\bx_t}{\wh{\bthet}_t}\right) - u\left(\ip{\bx_t}{\bthet}\right)\right)^2\xi_t \leq \bregd{\bthet}{\wh{\bthet}_1} - \bregd{\bthet}{\wh{\bthet}_{T+1}}\\
    &\phantom{=} + \lambda\sqrt{CBRT\log(6/\delta)} + \frac{\lambda^2 C^2}{\sigma}\sqrt{LCBRT\log(6/\delta)} + \frac{\lambda^2 C^2 B^2T}{2\sigma}.
\end{align*}
Let $\beta \in (0, 1)$ and take $\lambda = \min\left\{\frac{2\sigma(1-\beta)}{C^2BL}, \frac{1}{\sqrt{T}}\right\}$. Define $\beta' = 1 - \frac{C^2 L B}{2\sigma\sqrt{T}}$, and define $\bar{\beta} = \max\{\beta, \beta'\}$. Then $\frac{1}{L} - \frac{\lambda C^2 B}{2\sigma} = \frac{\bar{\beta}}{L} > 0$. Defining $h_t = \left(u\left(\ip{\bx_t}{\wh{\bthet}_t}\right) - u\left(\ip{\bx_t}{\bthet}\right)\right)^2\xi_t$, we find
\begin{align}
    \sum_{t=1}^Th_t &\leq \frac{L}{\bar{\beta}}\max\left\{\sqrt{T}, \frac{C^2 BL}{2\sigma(1-\beta)}\right\}\bregd{\bthet}{\wh{\bthet}_1} + \frac{L}{\bar{\beta}}\sqrt{CBRT\log(6/\delta)}\nonumber\\
    &\phantom{=} \hspace{.75in} + \frac{C^2 L}{\sigma\bar{\beta}}\sqrt{LCBR\log\left(6/\delta\right)} + \frac{C^2 L\sqrt{T}B^2}{2\sigma\bar{\beta}}.
    \label{eqn:online_noise_bound}
\end{align}
By Assumption~\ref{assmp:frei}, noting that $\norm{\bx_t}_* \leq C$ and $\norm{\wh{\bthet}_t} \leq R + \norm{\bthet}$, there exists a fixed $\gamma > 0$ such that $\sum_{t=1}^T\varepsilon_t \leq \frac{1}{\gamma}\sum_{t=1}^Th_t$. We now want to transfer this bound to a bound on $\varepsilon(h_t)$ via Lemma~\ref{lem:martingale}. Define $D_t^{(3)} = \varepsilon(h_t) - \varepsilon_t$, and note that this is a martingale difference sequence adapted to the filtration $\{\calF_t = \sigma(\bx_1, y_1, \bx_2, y_2, \hdots, \bx_t, y_t)\}$. $D_t^{(3)}$ satisfies the following inequalities almost surely,
\begin{align*}
    D_t^{(3)} &\leq \frac{1}{2}L^2 C^2 R,\\
    \Exp{\left(D_t^{(3)}\right)^2|\calF_{t-1}} &\leq \frac{1}{2}L^2 C^2 R\varepsilon(h_t).
\end{align*}
Thus, by Lemma~\ref{lem:martingale}, with probability at least $1-\delta/3$,
\begin{equation*}
    \sum_{\tau=1}^T\varepsilon(h_\tau) \leq \frac{L^2 C^2 R}{2(3-e)}\log(3/\delta) + \frac{1}{3-e}\sum_{\tau=1}^T\varepsilon_\tau.
\end{equation*}
Using \eqref{eqn:online_noise_bound}, we then have with probability at least $1-\delta$,
\begin{align*}
    \sum_{\tau=1}^T\varepsilon(h_\tau) &\leq \frac{L^2 C^2 R}{2(3-e)}\log(3/\delta) + \frac{L}{\bar{\beta}\gamma(3-e)}\max\left\{\sqrt{T}, \frac{C^2 BL}{2\sigma(1-\beta)}\right\}\bregd{\bthet}{\wh{\bthet}_1}\nonumber\\
    &\phantom{=} + \frac{L}{\bar{\beta}\gamma(3-e)}\sqrt{CBRT\log(6/\delta)} + \frac{C^2 L}{\sigma\bar{\beta}\gamma(3-e)}\sqrt{LCBR\log\left(6/\delta\right)} + \frac{C^2 L\sqrt{T}B^2}{2\sigma\bar{\beta}\gamma(3-e)}.
\end{align*}
Noting that $\min_{t < T}\varepsilon(h_\tau) \leq \frac{1}{T}\sum_{\tau=1}^T \varepsilon(h_\tau)$ completes the proof.
\end{proof}
\subsection{Proof of Theorem~\ref{thm:online_real}}
\begin{proof}
Again from \eqref{eqn:interim_step} adapted to the stochastic optimization setting, we have the bound
\begin{align*}
    \bregd{\bthet}{\wh{\bthet}_{t+1}} &\leq \bregd{\bthet}{\wh{\bthet}_t} + \frac{\lambda^2}{2\sigma}\left\Vert\left(u\left(\ip{\bx_t}{\wh{\bthet}_t}\right) - u\left(\ip{\bx_t}{\bthet}\right)\right)\xi_t\bx_t\right\Vert_*^2 \\
    &\phantom{=} \hspace{1.25in} + \lambda\ip{\left(u\left(\ip{\bx_t}{\wh{\bthet}_t}\right) - u\left(\ip{\bx_t}{\bthet}\right)\right)\xi_t\bx_t}{\bthet - \wh{\bthet}_t},\\
    &\leq \bregd{\bthet}{\wh{\bthet}_t} + \frac{\lambda^2C^2B}{2\sigma}\left(u\left(\ip{\bx_t}{\wh{\bthet}_t}\right) - u\left(\ip{\bx_t}{\bthet}\right)\right)^2\xi_t \\
    &\phantom{=} \hspace{1.25in} - \frac{\lambda}{L}\left(u\left(\ip{\bx_t}{\wh{\bthet}_t}\right)- u\left(\ip{\bx_t}{\bthet}\right)\right)^2\xi_t,\\
    &= \bregd{\bthet}{\wh{\bthet}_t} - \frac{\lambda}{L}\left(1 - \frac{\lambda LC^2B}{2\sigma}\right)\left(u\left(\ip{\bx_t}{\wh{\bthet}_t}\right) - u\left(\ip{\bx_t}{\bthet}\right)\right)^2\xi_t.
\end{align*}
Let $0 < \beta < 1$. Taking $\lambda = \frac{\left(1-\beta\right)2\sigma}{LC^2B}$,
\begin{align*}
    \bregd{\bthet}{\wh{\bthet}_{t+1}} &\leq \bregd{\bthet}{\wh{\bthet}_t} - \frac{2\sigma(1-\beta) \beta}{L^2C^2B}\left(u\left(\ip{\bx_t}{\wh{\bthet}_t}\right) - u\left(\ip{\bx_t}{\bthet}\right)\right)^2\xi_t,
\end{align*}
so that $\bregd{\bthet}{\wh{\bthet}_{t+1}} \leq \bregd{\bthet}{\wh{\bthet}_t} \leq \hdots \leq \bregd{\bthet}{\wh{\bthet}_1}$. Let $W$ be such that $\norm{\bthet} = W\sqrt{\frac{2\psi(\bthet)}{\sigma}}$. Then $\bregd{\bthet}{\wh{\bthet}_1} \leq \psi(\bthet)$ so that $\norm{\wh{\bthet}_t} \leq \left(1 + W\right)\sqrt{\frac{2\psi(\bthet)}{\sigma}}$ by $\sigma$-strong convexity of $\psi$ with respect to $\norm{\cdot}$. Summing both sides from $1$ to $T-1$ leads to the inequality
\begin{equation*}
    \bregd{\bthet}{\wh{\bthet}_T} \leq \bregd{\bthet}{\wh{\bthet}_1} - \frac{2\sigma(1-\beta) \beta}{L^2C^2B}\sum_{t=1}^{T-1}\left(u\left(\ip{\bx_t}{\wh{\bthet}_t}\right) - u\left(\ip{\bx_t}{\bthet}\right)\right)^2\xi_t.
\end{equation*}
Rearranging, using positivity of the Bregman divergence, and defining $h_t = \left(u\left(\ip{\bx_t}{\wh{\bthet}_t}\right) - u\left(\ip{\bx_t}{\bthet}\right)\right)^2\xi_t$, we conclude that
\begin{equation*}
    \sum_{t=1}^{T-1}h_t \leq \frac{L^2C^2B}{2\sigma(1-\beta)\beta}\bregd{\bthet}{\wh{\bthet}_1}.
\end{equation*}
Applying Assumption~\ref{assmp:frei} shows that there exists a $\gamma > 0$ such that
\begin{equation}
    \sum_{t=1}^{T-1}\varepsilon_t \leq \frac{L^2C^2B}{2\sigma(1-\beta)\beta\gamma}\bregd{\bthet}{\wh{\bthet}_1}.
    \label{eqn:online_bound}
\end{equation}
We would now like to transfer the bound \eqref{eqn:online_bound} to a bound on $\varepsilon(h_t)$. Define $D_t = \varepsilon\left(h_t\right) - \varepsilon_t$, and note that $\{D_t\}_{t=1}^\infty$ is a martingale difference sequence adapted to the filtration $\{\calF_t = \sigma(\bx_1, \bx_2, \hdots, \bx_t)\}_{t=1}^\infty$. Note that, almost surely,
\begin{align*}
    D_t &\leq \varepsilon(h_t) = \frac{1}{2}\Expsub{\bx \sim \calD}{\left(u\left(\ip{\wh{\bthet}_t}{\bx}\right) - u\left(\ip{\bthet}{\bx}\right)\right)^2}\\
    &\leq \frac{1}{2}L^2C^2\Vert\wh{\bthet}_t - \bthet\Vert^2 \leq \frac{L^2 C^2}{\sigma}\bregd{\bthet}{\wh{\bthet}_t} \leq \frac{L^2 C^2}{\sigma}\bregd{\bthet}{\wh{\bthet}_1}
\end{align*}
where we have applied $\sigma$-strong convexity of $\psi$ with respect to $\Vert\cdot\Vert$ to upper bound $\Vert\wh{\bthet}_t-\bthet\Vert^2$ by the corresponding Bregman divergence. 
Now, consider the following bound on the conditional variance
\begin{align*}
    \Exp{D_t^2|\calF_{t-1}} &= \Exp{\varepsilon(h_t)^2 - 2\varepsilon(h_t)\varepsilon_t + \varepsilon_t^2|\calF_{t-1}},\\
    &= \varepsilon(h_t)^2 - 2\varepsilon(h_t)^2 + \Exp{\varepsilon_t^2|\calF_{t-1}},\\
    &\leq \Exp{\varepsilon_t^2|\calF_{t-1}},\\
    &= \frac{1}{4}\Exp{\left(u\left(\ip{\bx_t}{\wh{\bthet}_t}\right) - u\left(\ip{\bx_t}{\bthet}\right)\right)^4|\calF_{t-1}},\\
    &\leq \frac{L^2 C^2 \bregd{\bthet}{\wh{\bthet}_1}}{\sigma}\varepsilon(h_t).
\end{align*}
Hence by Lemma~\ref{lem:martingale}, with probability at least $1-\delta$,
\begin{align*}
    \sum_{\tau=1}^t \left(\varepsilon(h_\tau) - \varepsilon_\tau\right) &\leq \frac{L^2 C^2}{\sigma}\bregd{\bthet}{\wh{\bthet}_1}\log(1/\delta) + (e-2)\sum_{\tau=1}^t \varepsilon(h_\tau).
\end{align*}
Rearranging terms,
\begin{align*}
    (3-e)\sum_{\tau=1}^t \varepsilon(h_\tau) &\leq \frac{L^2 C^2}{\sigma}\bregd{\bthet}{\wh{\bthet}_1}\log(1/\delta) + \sum_{\tau=1}^t \varepsilon_\tau.
\end{align*}
Applying the bound from \eqref{eqn:online_bound},
\begin{align*}
    (3-e)\sum_{\tau=1}^t \varepsilon(h_\tau) &\leq \frac{L^2 C^2}{\sigma}\bregd{\bthet}{\wh{\bthet}_1}\left(\log(1/\delta) + \frac{1}{2(1-\beta)\beta\gamma}\right).
\end{align*}
We then conclude
\begin{equation*}
    \min_{t < T} \varepsilon(h_t) \leq \frac{1}{T}\sum_{\tau=1}^T \varepsilon(h_\tau) \leq \frac{L^2 C^2 \bregd{\bthet}{\wh{\bthet}_1}}{\sigma(3-e)T}\left(\log(1/\delta) + \frac{B}{2(1-\beta)\beta\gamma}\right),
\end{equation*}
which completes the proof.
\end{proof}